\documentclass[twoside,11pt]{article}

\usepackage{jmlr2e}

\usepackage{booktabs}
\usepackage{float}
\usepackage[load-configurations=version-1]{siunitx}
\usepackage{bm}

\usepackage{amsmath}
\usepackage{graphicx}
\usepackage{color}
\usepackage{amssymb}
\usepackage{bm}
\usepackage[linesnumbered,ruled,vlined]{algorithm2e}
\usepackage{float}
\usepackage{mathtools}
\usepackage{wrapfig}
\usepackage[export]{adjustbox}[2011/08/13]

\DeclareMathOperator*{\argmin}{arg\,min}
\DeclareMathOperator*{\argmax}{arg\,max}

\jmlrheading{1}{2018}{1-48}{4/00}{10/00}{Omid Shams Solari, James B. Brown and Peter J. Bickel}

\ShortHeadings{MuLe}{Solari, Brown and Bickel}
\firstpageno{1}

\begin{document}

\title{Sparse Canonical Correlation Analysis via Concave Minimization}

\author{\name Omid Shams Solari \email solari@berkeley.edu \\
       \addr Department of Statistics\\
       University of California, Berkeley
       \AND
       \name James B. Brown \email jbbrown@lbl.gov\\
       \addr Lawrence Berkeley National Laboratory and Department of Statistics\\
       University of California, Berkeley
       \AND
       \name Peter J. Bickel \email 
       bickel@stat.berkeley.edu\\
       \addr Department of Statistics\\
       University of California, Berkeley}

\editor{}

\makeatletter
\let\mule\@title

\maketitle

\begin{abstract}


A new approach to the \textit{sparse Canonical Correlation Analysis (sCCA)} is proposed with the aim of discovering interpretable associations in very high-dimensional multi-view, i.e. observations of multiple sets of variables on the same subjects, problems. Inspired by the sparse PCA approach of \cite{journe:nesterov}, we also show that the sparse CCA formulation, while non-convex, is equivalent to a maximization program of a convex objective over a compact set for which we propose a first-order gradient method. This result helps us reduce the search space drastically to the boundaries of the set. Consequently, we propose a two-step algorithm, where we first infer the sparsity pattern of the canonical directions using our fast algorithm, then we shrink each view, i.e. observations of a set of covariates, to contain observations on the sets of covariates selected in the previous step, and compute their canonical directions via any CCA algorithm.  We also introduce \textit{Directed Sparse CCA}, which is able to find associations which are aligned with a specified experiment design, and \textit{Multi-View sCCA} which is used to discover associations between multiple sets of covariates. Our simulations establish the superior convergence properties and computational efficiency of our algorithm as well as accuracy in terms of the canonical correlation and its ability to recover the supports of the canonical directions. We study the associations between metabolomics, trasncriptomics and microbiomics in a multi-omic study using \texttt{MuLe}, which is an \texttt{R} package that implements our approach, in order to form hypotheses on mechanisms of adaptations of \textit{Drosophila Melanogaster} to high doses of environmental toxicants, specifically Atrazine, which is a commonly used chemical fertilizer.

\end{abstract}


\section{Introduction}

\textit{Canonical Correlation Analysis}(CCA), \cite{hotelling} , is a powerful set of approaches for analyzing the relationship between two sets of random vectors, and discovering associations between elements of said vectors. Classical CCA is specifically concerned with finding linear combinations of the elements of each random vector such that they are maximally correlated estimated using observations of each random vector on matching subjects/individuals, i.e. different \textit{views}, of the same latent random vector. In this article, we use the terms \textit{view} and \textit{dataset} interchangeably, denoted by $\bm{X}_i \in \mathbb{R}^{n \times p_i}$, to refer to $n$ observations of a random vector of length $p_i$.

CCA has been widely used in various fields of data science and machine learning and has found successful applications in finance, neuro-imaging, computer vision, NLP, social sciences, geography, collaborative filtering, astronomy and a new surge in genomics, especially in recently popular multi-assay genetic/clinical population studies.
After its proposition by \cite{hotelling}, CCA was first applied in \cite{waugh42} where he studied the relationship between the characteristics of wheat and the resulting flour. He demonstrated that desirable wheat is high in texture, density and protein content and low on damaged kernels and foreign materials. Other rather classic applications of CCA include: medical geography, where \cite{monmonier73} showed direct association between the number of hospital beds per capita and physician ratios, socio-medical studies, e.g. \cite{hopkins69} studies the relationship between housing and health in Baltimore, education, \cite{dunham75} analyzes the association between measures of academic performance in college and exam scores in high school, economics, where \cite{assetLiability} employs this technique to identify and describe hedging behavior between the asset side and the capital side of the balance sheets of a selection of US. banks, signal processing, e.g. \cite{schell95} introduces \textit{Programmable CCA} to design filters to distinguish between desired signal and noise, time-series analysis, e.g. \cite{heij91} employs CCA for state-space modeling, geography, e.g. \cite{2001JHyd..254..157O} perform a regional flood frequency analysis using CCA by investigating the correlation structure between watershed characteristics and flood peaks, medical imaging, e.g. \cite{fMRI} benefited from CCA in detecting activated brain regions based on physiological parameters such as temporal shape and delay of the hemodynamic response. There are plenty of other examples in the fields of chemistry, e.g. \cite{tu89}, physics, e.g. \cite{wong80}, dentistry, e.g. \cite{lindsey85} where CCA is utilized to discover complex yet meaningful associations between two sets of variables.

CCA and its variants have also found substantial grounds in modern fields of research such as artificial intelligence and statistical learning, neuro-imaging and human perception, context-based content retrieval, collaborative filtering, dimensionality reduction and feature selection, and spatial and temporal genome-wide association studies. \cite{cao15} and \cite{nakanishi15} used CCA in the area of Brain Computer Interface(BCI) to recognize the frequency components of target stimuli. In the area of image recognition, \cite{hardoon2004canonical} use a kernel CCA method to perform content-based image retrieval and learn semantics of multimedia content by combining image and text data. \cite{ogura13}, \cite{shen13}, and \cite{wang13} have employed CCA and its variants for the purpose of feature selection/extraction/fusion and dimensionality reduction.

Modern Canonical Correlation Analysis algorithms have had a significant surge in genomics esp. multi-omic genetic and environmental studies in the last few years mainly due to fast and efficient genome sequencing and measurement technologies becoming more accessible. Such studies typically involve two or more, usually high-dimensional, omic datasets, e.g. trascriptomic, metabolomic, microbiomic data. An instance of such study is \cite{hyman2002} where they performed CGH analysis on cDNA microarrays in breast cancer and compared copy number and mRNA expression levels to infer the impact of genomic  changes  on  gene  expression. \cite{genomicYamanishi} successfully utilized this method to recognize the operons in \textit{Escherichia Coli} genome by comparing three datasets corresponding to functional, locational and expression relationships between the genes.  \cite{morley2004}, \cite{pollack2002}, \cite{snijders2017influence}, \cite{orsini2018early}, \cite{fang16}, \cite{rousu13}, \cite{seoane14}, \cite{baur15}, \cite{sarkar15}, and \cite{cichonska16} are few other notable relevant works.

In the next section we provide an overview of the common approaches, but we first compile the notation used throughout the paper in the subsection below.

\section{Notation}

Each view, i.e. the observation matrix on random vector $X_i(\omega): \Omega \rightarrow \mathbb{R}^{p_i}$, is denoted by $\bm{X}_i \in \mathbb{R}^{n \times p_i}$, $i = 1, \ldots, m$. $n$ is reserved to denote the sample size and $p_i$ to denote the length of each random vector $X_i, i = 1, \ldots, m$. Canonical directions are denoted by $\bm{z}_i \in \mathcal{B}^{p_i}$, or $\bm{z}_i \in \mathcal{S}^{p_i}$, and $\bm{Z}_i \in \mathcal{S}_d^{p_i}$, where $\mathcal{B} = \{\bm{x} \in \mathbb{R} | \| \bm{x}\|_2 \leq 1 \}$ and $\mathcal{S} = \{\bm{x} \in \mathbb{R} | \| \bm{x}\|_2 = 1 \}$. $l_x(\bm{z}) = \| \bm{z} \|_x: \mathbb{R}^{p} \rightarrow \mathbb{R}$ denotes any norm function, more specifically $l_{0/1}(\bm{z}) = \| \bm{z} \|_{0/1}$, and $\bm{\tau}^{(i)}$ refers to the $i-th$ non-zero element of the vector which is specifically used for the sparsity pattern vector. Sample covariance matrices corresponding to the $i$-th and $j$-th views is denoted by $\bm{C}_{ij}$. We drop the subscript when we only have two views. $max(x, 0)$ is also denoted by $[x]_+$. We also coin the term \textit{accessory variables} in Section \ref{subsec:directed} to refer to the variables towards which we direct estimated canonical directions, disregarding their causal roles as covariates or dependent variables. We also use ``program" to refer to ``optimization programs".

\section{An Overview of Approaches to the CCA Problem}
\label{sec:pDefinition}

This subsection covers a literature review of Canonical Correlation Analysis, common approaches, and their statistical assumptions and approximations. While linear approaches and especially their regularized extensions are the main focus of this paper, we have also provided an overview of non-linear approaches, e.g. kernelized model of \cite{lai00} and DeepCCA of \cite{andrew13}.

\subsection{CCA}
\label{subsec:cca}

Let $X(\omega): \Omega \rightarrow \mathbb{R}^p$ be a random vector with covariance matrix $\bm{\Sigma} \in \mathbb{R}^{p \times p}$. Further assume that $\mathbb{E}X = \mathbf{0}$. Now partition $X$ into $X_1 \in \mathbb{R}^{p_1}$ and $X_2 \in \mathbb{R}^{p_2}$.
The covariance matrix can be partitioned accordingly.
\begin{equation}
\label{eq:covariance}
\bm{\Sigma}
=
\begin{bmatrix}
\bm{\Sigma_{11}} & \bm{\Sigma_{12}}\\
\bm{\Sigma_{21}} & \bm{\Sigma_{22}}
\end{bmatrix}
\end{equation}

\textit{Canonical Correlation Analysis}, \cite{hotelling}, identifies two weight vectors $\bm{z}_1$ and $\bm{z}_2$ such that the Pearson correlation coefficient between the images $X_1\bm{z_1}$ and $X_2\bm{z_2}$ is maximized,

\begin{equation}
\label{eq:opt}
\begin{split}
    \rho(\bm{z}_1^*, \bm{z}_2^*) =& \max_{\bm{z}_1 \in \mathbb{R}^{p_1}, \bm{z}_2 \in \mathbb{R}^{p_2}} \frac{\mathbb{E}[(X_1\bm{z_1})^{\top}X_2\bm{z_2}]}{\mathbb{E}[(X_1\bm{z_1})^2]^{1/2}\mathbb{E}[(X_2\bm{z_2})^2]^{1/2}} \\=& \max_{\bm{z}_1 \in \mathbb{R}^{p_1}, \bm{z}_2 \in \mathbb{R}^{p_2}}\frac{\bm{z}_1^{\top} \bm{\Sigma}_{12} \bm{z}_2 }{\sqrt{\bm{z}_1^{\top} \bm{\Sigma}_{11} \bm{z}_1} \sqrt{\bm{z}_2^{\top} \bm{\Sigma}_{22}\bm{z}_2}}\\
    =& \max_{\substack{\bm{z}_1 \in \mathbb{R}^{p_1}, \bm{z}_2 \in \mathbb{R}^{p_2}\\ \bm{z}_1^T\bm{\Sigma}_{11}\bm{z}_1 = 1\\\bm{z}_2^T\bm{\Sigma}_{22}\bm{z}_2 = 1}} \bm{z}_1^T\bm{\Sigma}_{12}\bm{z}_2
\end{split}
\end{equation}

where the last line is due to scale-invariability of $\rho$.



The images $X_1\bm{z}_1$ and $X_2\bm{z}_2$ are called the \textit{canonical variables} and the weights $\bm{z}_1$ and $\bm{z}_2$ are the \textit{canonical loading vectors} or the \textit{canonical directions}. The loading vectors $(\bm{z}_1^{(1)}, \bm{z}_2^{(1)})$ obtained from optimizing Program \ref{eq:opt} reveal the first canonical correlation. $(\bm{z}_1^{(2)}, \bm{z}_2^{(2)})$ that maximize \ref{eq:opt} but with an added constraint that their corresponding images are respectively orthogonal to the first pair determine the second canonical correlation. This procedure is continued until no more pairs are found. The number $r \leq min\{p_1, p_2\}$ of pairs of canonical variables can be interpreted as the number of patterns in the correlation structure.



We estimate the population parameters by plugging in sample estimates of the expectations in Program \ref{eq:opt}. With $\bm{X}_1 \in \mathbb{R}^{n \times p_1}$ and $\bm{X}_2 \in \mathbb{R}^{n \times p_2}$ being the sample matrices corresponding to $X_1$ and $X_2$ respectively, $\bm{\Sigma}_{ij}, i,j \in \{1,2\}$ is estimated by the sample covariance matrices $\bm{C}_{ij} = \frac{1}{n}\bm{X}_i^{\top}\bm{X}_j, i,j \in \{1,2\}$.

Therefore the sample CCA optimization problem may be written as,

\begin{equation}
\label{eq:sampOpt}
\max_{\substack{\bm{z}_1 \in \mathbb{R}^{p_1}, \bm{z}_2 \in \mathbb{R}^{p_2}\\ \bm{z}_1^{\top}\bm{C}_{11}\bm{z}_1 = 1\\\bm{z}_2^{\top}\bm{C}_{22}\bm{z}_2 = 1}} \bm{z}_1^{\top}\bm{C}_{12}\bm{z}_2
\end{equation}

Generally, this optimization problem is solved using one of the three classes of techniques. \cite{hotelling} solves this problem using Lagrange multipliers to obtain the characteristic equation which is a \textit{standard eigenvalue problem}, 

\begin{equation}
\label{eq:sep}
\bm{C}_{22}^{-1}\bm{C}_{21}\bm{C}_{11}^{-1}\bm{C}_{12}^{-1}\bm{z}_2 = \rho^2 \bm{z}_2
\end{equation}

\cite{bach02} and \cite{hardoon2004canonical} form the following system of equations using the same Lagrange multiplier technique,

\begin{equation}
\label{gep}
\begin{pmatrix} 0 & \bm{C}_{12} \\ \bm{C}_{21} & 0 \end{pmatrix} \begin{pmatrix} \bm{z}_1 \\ \bm{z}_2 \end{pmatrix} = \rho \begin{pmatrix} \bm{C}_{11} & 0 \\ 0 & \bm{C}_{22} \end{pmatrix} \begin{pmatrix} \bm{z}_1 \\ \bm{z}_2 \end{pmatrix}
\end{equation}

Which can be regarded as a \textit{generalized eigenvalue problem} and the positive generalized eigenvalues as the squared canonical correlations.

\cite{healy57} and \cite{ewerbring89} used \textit{singular value decomposition} to find canonical correlations. In this approach, inverse square roots of the sample covariance matrices $\bm{C}_{11}^{-1/2}$ and $\bm{C}_{22}^{-1/2}$ are computed. Canonical loading vectors are computed using the following SVD,

\begin{equation}
\label{eq:svd}
\bm{C}_{11}^{-1/2} \bm{C}_{12} \bm{C}_{22}^{-1/2} = \bm{U}\bm{D}\bm{V}^{\top}
\end{equation}

Where $\bm{U}$ and $\bm{V}$ are orthonormal matrices and  the non-zero elements of the diagonal matrix $D$ correspond to the singular values which are equal to the canonical correlations. $\bm{z}_1^{(k)}$ and $\bm{z}_2^{(k)}$ are obtained using $\bm{C}_{11}^{-1/2}\bm{U}_{.k}$ and $\bm{C}_{22}^{-1/2}\bm{V}_{.k}$ respectively.

\subsection{Regularized CCA}

Techniques reviewed above are applicable in over-determined systems or low-dimensional regimes. However, in high-dimensional regimes where there are fewer observations than variables, $n \leq max\{p_1, p_2\}$, new approaches are needed to overcome the issues of singular covariance matrices and overfitting as well as lack of identifiability of original parameter. These approaches are also helpful in reducing the estimation variance, providing robustness to outliers, and, of special relevance to this paper, offering more interpretable models.

\subsubsection{Ridge Regularization}

So called \textit{canonical ridge} was proposed in \cite{vinod76} to address the problem of insufficient sample size. Here, the innvertibility of the sample covariance matrices $C_{11}$ and $C_{22}$ is improved by introducing ridge penalties, which comes at the cost of introducing two more hyper-parameters, $c_1, c_2 \geq 0$. Ultimately, the optimization constraints in Program \ref{eq:sampOpt} become

\begin{equation}
\label{eq:regularization}
\begin{split}
z_1^{\top}(C_{11} + c_1I)z_1 =& 1\\
z_2^{\top}(C_{22} + c_2I)z_2 =& 1
\end{split}
\end{equation}

Any of the three algorithms of Section \ref{subsec:cca} may be modified for solving this problem.

\subsubsection{Lasso Regularization}
\label{subsub:lasso}

LASSO or $L_1$ regularized CCA, which is one of the two main foci of this paper, is specifically useful when there are not nearly as many observations as covariates. In such high-dimensional settings ridge-regularized methods, although successfully reducing instability, lack interpretability and overfitting is still an issue. 
To this end, a school of methods exist which does both variable selection and estimation simultaneously or sequentially through sparsity inducing regularization. \cite{parkhomenkoGWSCCA}, \cite{parkhomenkoSCCA} , and \cite{witten:tibshirani:2009}
advise a simple soft-thresholding algorithm to enforce sparsity. They apply \textit{sparse CCA} methods to find meaningful associations between genomic datasets, be it RNA expression datasets, single-loci DNA modifications or regions of loss/gain within the genome. \cite{waaijenborg} incorporates a combination of $L_1$ and $L_2$ penalties into the CCA model to identify gene networks that are influenced by multiple genetic changes. \cite{hardoon11} offers a different formulation using convex least squares. In their approach the association between the linear combination of one view and the Gram matrix of the other view is computed. They demonstrate that in cases when the observations are very high-dimensional, their sparse CCA approach outperforms KCCA significantly.

The approaches to the $L_1$ regularized CCA proposed in the literature referenced above are almost identical, except for that of \cite{hardoon11}. Despite small differences, e.g. \cite{waaijenborg} uses elastic net which is a mixture of LASSO and ridge penalties, they all solve a regularized SVD using alternating maximization of slightly different optimization programs. \textit{Penalized Matrix Decomposition(PMD)} algorithm which was first introduced in \cite{wittentibhastie09}, then extended in \cite{witten:tibshirani:2009} estimates the sample covariance matrix $\bm{C}_{12}$ with closest rank-one matrix in a Frobenius norm sense under some constraints.

\begin{equation}
\label{eq:sCCA}
\begin{split}
(\bm{z}_1^*,\bm{z}_2^*) &= \argmin_{\substack{ \bm{z}_1 \in \mathcal{B}^{p_1}, \bm{z}_2 \in \mathcal{B}^{p_2}\\ \| \bm{z}_1\|_1 \leq c_1, \| \bm{z}_2\|_1 \leq c_2, \sigma \geq 0}} \| \bm{C}_{12} - \sigma \bm{z}_1\bm{z}_2^{\top} \|^2_F = \argmax_{\substack{ \bm{z}_1 \in \mathcal{B}^{p_1}, \bm{z}_2 \in \mathcal{B}^{p_2}\\ \| \bm{z}_1\|_1 \leq c_1, \| \bm{z}_2\|_1 \leq c_2}} \bm{z}_1^{\top}\bm{C}_{12}\bm{z}_2
\end{split}
\end{equation}

where $c_i \geq 0, i = 1,2$ are sparsity parameters. The last statement in Program \ref{eq:sCCA} is of course a penalized SVD.

\subsubsection{Cardinality Regularization}

Most approaches to the sparse CCA problem involve the LASSO regularization which was reviewed in Section \ref{subsub:lasso}. However, few greedy approaches were also developed cardinality or $L_0$ regularized case.

\begin{equation}
\label{eq:sCCAl0}
\begin{split}
(\bm{z}_1^*,\bm{z}_2^*) =  \argmax_{\substack{ \bm{z}_1 \in \mathcal{B}^{p_1}, \bm{z}_2 \in \mathcal{B}^{p_2}\\ \| \bm{z}_1\|_0 \leq c_1, \| \bm{z}_2\|_0 \leq c_2}} \bm{z}_1^{\top}\bm{C}_{12}\bm{z}_2
\end{split}
\end{equation}

where as before the sparsity parameters are non-negative. \cite{wiesel2008greedy} develop a greedy algorithm which is based on the sparse PCA approach of \cite{d2008optimal}, which we also base our $L_0$ regularized algorithm on, and demonstrate the effectiveness of their backward greedy algorithm in high-dimensional settings.

\subsection{Bayesian CCA}

Bayesian approaches to CCA were introduced to increase the robustness of the model in low sample size scenarios and improve the validity of the model by allowing different distributions. \cite{klami12} offer a detailed review of Bayesian approaches to CCA, and \cite{bach05} offer a formalization of this problem within a probabilistic framework. In these models latent variables $U \sim \mathcal{N}(0, I_l)$ where $l \leq min\{p_1,p_2\}$ are assumed to generate the observations $\bm{x}_1^{(i)} \in \mathbb{R}^{p_1}$ and $\bm{x}_2^{(i)} \in \mathbb{R}^{p_2}$ through

\begin{equation}
\label{eq:lvar}
\begin{split}
X_1|U &\sim \mathcal{N}(\bm{S}_1U + \bm{\mu}_1, \bm{\Psi}_1)\\
X_2|U &\sim \mathcal{N}(\bm{S}_2U + \bm{\mu}_2, \bm{\Psi}_2)
\end{split}
\end{equation}

where $\bm{S}_1$ and $\bm{S}_2$ are transform matrices and $\bm{\Psi}_1$ and $\bm{\Psi}_1$ noise covariance matrices. Maximum likelihood estimates of model parameters are used to estimate the posterior expectation of $U$.

\subsection{Non-Linear Transformations}

So far, our discussion of CCA and its extensions were constrained to linear transformations of observed random variables. Analyzing non-linear correlation structures, however, requires further innovation. \textit{(Deep) neural networks(DNN)} based CCA and \textit{kernel CCA} are reviewed as the two main schools of methods for uncovering non-linear canonical correlations.

\subsubsection{DNN-Based CCA}

\cite{lai99} used neural networks to find non-linear canonical correlation and detect shift information in a random dot stereogram data. \cite{lai00} extends this by adding a non-linearity to their network and also by non-linearly transforming the data to a feature space and then performing linear CCA. \cite{andrew13} developed the package \texttt{deepCCA}, which will be explained here briefly. In this approach, each dataset, $\bm{X}_i$, is transformed through multiple layers by applying sigmoid functions on linear transformation of the input to the layer $j = 1, \ldots, J$ of network $i = 1, \ldots, I$, 

\begin{equation}
\bm{a}_i^j = \sigma (\bm{Z}_i^j\bm{x}_i + \bm{b}_i^j), \quad i = 1, \ldots , I, j= 1, \ldots , J
\end{equation}

 where $\sigma$ is a nonlinear sigmoid function and $\bm{Z}_i^j$ and $\bm{b}_i^j$ are the weight matrices and bias vectors respectively that need to be learned such that some cost function is minimized. The cost function they defined was the correlation between the output views of all $I$ datasets. Assuming output matrices $\bm{H}_1 \in \mathbb{R}^{o\times n}$ and $\bm{H}_2 \in \mathbb{R}^{o\times n}$, define $\bm{C}_{12} = \frac{1}{n-1}\tilde{\bm{H}}_1\tilde{\bm{H}}_2^{\top}$, $\bm{C}_{11} = \frac{1}{n-1}\tilde{\bm{H}}_1\tilde{\bm{H}}_1^{\top} + \gamma_1 \bm{I}$ and $\bm{C}_{22} = \frac{1}{n-1}\tilde{\bm{H}}_2\tilde{\bm{H}}_2^{\top} + \gamma_2 \bm{I}$, where $\tilde{\bm{H}}_i = \bm{H}_i - \frac{1}{n} \bm{H}_i \bm{1}$ are the centered output matrices. Also define $\bm{T} = \bm{C}_{11}^{-1/2}\bm{C}_{12}\bm{C}_{22}^{-1/2}$. Then the correlation objective to be maximized can be written as the trace norm of $\bm{T}$.
 
\begin{equation}
corr(\bm{H}_1,\bm{H}_2) = tr(\bm{T}^{\top}\bm{T})^{1/2}
\end{equation}

Obviously $H_i = f(\bm{z}_i^j, b_i^j), j = 1, \ldots, J$.

Using \textit{DNN}s for multi-view learning is a very active line of research. Recently, models based on \textit{Variational Auto-Encoders(VAE)} have become popular[\cite{wang2016deep}].

\subsubsection{Kernel CCA \& The Kernel Trick}

\textit{Kernel} methods are more popular for analyzing non-linear associations[\cite{lai00}]. This is for the most part due to the vast theoretical literature on kernel methods, mainly from SVM literature, [\cite{gestel01}; \cite{cai13}; \cite{blaschko08}; \cite{hardoon09}; \cite{alam08}] and part due to the significantly fewer number of parameters to be estimated compared to DNNs[\cite{akaho01}]. \cite{melzer01} applies non-linear feature extraction to object recognition and compares it to non-linear PCA. \cite{bach02} uses CCA based methods in kernel Hilbert spaces for \textit{Independent Component Analysis(ICA)} and present efficient computation of their derivatives. \cite{larson14} utilizes kernel CCA to discover complex multi-loci disease-inducing SNPs related to ovarian cancer.

Kernelized methods use non-linear mappings,$\phi_1(\bm{X}_1)$ and $\phi_2(\bm{X}_2)$, of observations to non-Euclidean spaces, $\mathcal{H}_1$ and $\mathcal{H}_1$, where the measures of similarity between images are no longer linear. The similarity may be captured by a symmetric positive semi-definite kernel, which corresponds to the inner product in Hilbert spaces. In essence, KCCA first transforms the observations into Hilbert spaces $\mathcal{H}_1$ and $\mathcal{H}_2$ using PSD kernels,

\begin{equation}
\label{eq:kernel}
k_1(\bm{x}_{1i}, \bm{x}_{1j}) = \langle \phi_1(\bm{x}_{1i}), \phi_1(\bm{x}_{1j}) \rangle_{\mathcal{H}_1}, \quad k_2(\bm{x}_{2i}, \bm{x}_{2j}) = \langle \phi_2(\bm{x}_{2i}), \phi_2(\bm{x}_{2j}) \rangle_{\mathcal{H}_2}
\end{equation}

In practice, we don't need to specify the mappings $\phi_i(\bm{x}_{i,j})$. \textit{Mercer's theorem}[\cite{mercer1909xvi}] guarantees that as long as $k_1(\bm{x}_{ij}, \bm{x}_{ij}')$ is a positive semi-definite inner-product kernel, there is a corresponding $\phi_i: \mathbb{R}^{p_i} \rightarrow  \mathcal{H}$ equipped with inner-product $<.,.>_{\mathcal{H}}$. This permits us to bypass evaluating $\phi_i$ and go straight to evaluating inner-product kernels $k_i, 1, \ldots, I$. The rest of the analysis will be quite similar to the CCA problem except that the observation matrices $\bm{X}_i$ are replaced by their corresponding Gram matrices $K_i$ for $i= 1,\ldots,I$. For a more comprehensive treatment, refer to \cite{hardoon2004canonical} and \cite{bach02}.


The remainder of this paper is organized as follows: In Section \ref{sec:convex} we introduce the optimization problems corresponding to $L_0$/$L_1$\textit{regularized CCA} which are then extended to \textit{Multi-View Sparse CCA} and \textit{Directed Sparse CCA} in Section \ref{sec:ext}. In Section \ref{sec:mule}, we propose algorithms that solve the optimization programs of Sections \ref{sec:convex} and \ref{sec:ext}. In Section \ref{sec:sim} we apply \texttt{MuLe}, the \texttt{R}-package that implements our algorithms, to simulated data, where we benchmark our method and also compare it to several other available approaches. We also utilize it in Section \ref{sec:reda} to discover and interpret multi-omic associations which explain the mechanisms of adaptations of \textit{Dropsophila Melanogaster} to environmental pesticides. We conclude this paper in Chapter \ref{sec:conclusion}. Appendices are referenced in the text wherever applicable.


\section{Sparse Canonical Correlation Analysis}
\label{sec:convex}

We consider \textit{sparse CCA} formulations of the following form,

\begin{equation}
\label{eq:ref}
\phi_{l_x, l_x}(\gamma_1, \gamma_2) = \max_{\bm{z}_1 \in \mathcal{B}^{p_1} } \max_{ \bm{z}_2 \in \mathcal{B}^{p_2} } \bm{z}_1^T\bm{C}_{12}\bm{z}_2 - \gamma_1 l_x( \bm{z}_1 ) - \gamma_2 l_x( \bm{z}_2 ) 
\end{equation}

where $l_x = l_x(\bm{z})$ is a sparsity-inducing norm function, $\gamma_i \geq 0$, $i = 1,2$ are regularization parameters, and $\bm{C}_{12}= 1/n \bm{X}_1^{\top}\bm{X}_2$ is the sample covariance matrix.

\subsection{\texorpdfstring{$L_1$}{TEXT} Regularization}
\label{sec:l1reg}



Consider $x = 1$ in Program \ref{eq:ref},

\begin{equation}
\label{eq:l1reg}
\phi_{l_1, l_1}(\gamma_1, \gamma_2) = \max_{\bm{z}_1 \in \mathcal{B}^{p_1} } \max_{ \bm{z}_2 \in \mathcal{B}^{p_2} } \bm{z}_1^T\bm{C}_{12}\bm{z}_2 - \gamma_1 \| \bm{z}_1 \|_1 - \gamma_2 \| \bm{z}_2 \|_1 
\end{equation}

This optimization program is equivalent\footnote{Optimization programs $\psi_{\bm{x}}(\bm{\lambda} )$ and $\eta_{\bm{y}}(\bm{\mu})$ are called \textit{equivalent} if there is a one-to-one mapping $g:\mathcal{D}_{\bm{\lambda}} \rightarrow \mathcal{D}_{\bm{\mu}}$ such that $\bm{x}^* = \bm{y}^*$ if $\bm{\lambda} = g(\bm{\mu})$.} to the one in \ref{eq:sCCA}.


\begin{theorem}
\label{thm:l1}
Maximizers, $(\bm{z}_1^*, \bm{z}_2^*)$, of $\phi_{l_1, l_1}(\gamma_1, \gamma_2)$ in Program \ref{eq:l1reg} are given by,

\begin{equation}
\label{eq:z1}
\bm{z}_1^* = \argmax_{\bm{z}_1 \in \mathcal{B}^{p_1}} \sum_{i= 1}^{p_2} [ | \bm{c}_i^T \bm{z}_1 | - \gamma_2  ]_+^2 - \gamma_1 \|\bm{z}_1\|_1
\end{equation}

and

\begin{equation}
\label{eq:zstar}
z_{2i}^* = z_{2i}^*(\gamma_2) = \frac{sgn( \bm{c}_i^T \bm{z}_1 ) [ | \bm{c}_i^T \bm{z}_1 | - \gamma_2  ]_+}{\sqrt{\sum_{k=1}^{p_2} [|\bm{c}_k^T \bm{z}_1 | - \gamma_2]_+^2 } }, \quad i = 1, \ldots, p_2.
\end{equation}
\end{theorem}

\begin{proof}\footnote{We use the technique introduced in \cite{journe:nesterov} for sparse PCA to carry out the proofs of Theorems \ref{thm:l1} and \ref{thm:l0}}
\begin{align}
\label{eq:reformulation}
\begin{split}
\phi_{l_1, l_1}(\gamma_1, \gamma_2) &= \max_{\bm{z}_1 \in \mathcal{B}^{p_1} } \max_{ \bm{z}_2 \in \mathcal{B}^{p_2} } \bm{z}_1^{\top}\bm{C}_{12}\bm{z}_2 - \gamma_1 \| \bm{z}_1 \|_1 - \gamma_2 \| \bm{z}_2 \|_1 \\
&= \max_{\bm{z}_1 \in \mathcal{B}^{p_1} } \max_{ \bm{z}_2 \in \mathcal{B}^{p_2} }  \sum_{i = 1}^{p_2} z_{2i}(\bm{c}_i^{\top} \bm{z}_1) - \gamma_2 \| \bm{z}_2 \|_1 -\gamma_1 \| \bm{z}_1 \|_1\\
&= \max_{\bm{z}_1 \in \mathcal{B}^{p_1} } \max_{ \bm{z}_2' \in \mathcal{B}^{p_2} }  \sum_{i = 1}^{p_2} |z'_{2i}| (| \bm{c}_i^{\top} \bm{z}_1 | - \gamma_2 ) - \gamma_1 \| \bm{z}_1 \|_1
\end{split}
\end{align}

where we used the following change-of-variable $\bm{z}_{2i} = sgn(\bm{c}_i^{\top}\bm{z}_1)\bm{z}_{2i}'$. We optimize \ref{eq:reformulation} for $\bm{z}_2'$ for fixed $\bm{z}_1$ and change it back to $\bm{z}_2$ to get the result in Equation \ref{eq:zstar}. Substituting this result back in \ref{eq:reformulation},

\begin{equation}
\label{eq:z1l1}
\phi_{l_1,l_1}^2(\gamma_1, \gamma_2) = \argmax_{\bm{z}_1 \in \mathcal{B}^{p_1}} \sum_{i= 1}^{p_2} [ | \bm{c}_i^T \bm{z}_1 | - \gamma_2  ]_+^2 - \gamma_1 \|\bm{z}_1 \|_1
\end{equation}

\end{proof}





The following corollary asserts that we can provide the necessary and sufficient conditions based on the solution $\bm{z}_1^*$ in order to find the sparsity pattern of $\bm{z}_2^*$, i.e. $supp(\bm{z}_2^*)$, denoted in this paper as  $\bm{\tau}_2 \in \{0,1\}^{p_2}$.

\begin{corollary}
\label{cor:sparsity}
Given the sparsity parameter $\gamma_2$ and maximizer $\bm{z}_1^*$ of the program \ref{eq:z1l1}, entries $z_{2i}^*$, refer to \ref{eq:zstar}, for which $|\bm{c}_i^{\top}\bm{z}_1^*| \leq \gamma_2$ are identically zero.
\end{corollary}

\begin{proof}
According to Equation \ref{eq:zstar} of Theorem \ref{thm:l1}, 

\begin{equation}
\label{eq:sparsityl1}
  z_{2i}^* = 0 \Leftrightarrow [ | \bm{c}_i^T \bm{z}_1^* | - \gamma_2  ]_+  = 0 \Leftrightarrow |\bm{c}_i^T \bm{z}_1^* | \leq \gamma_2
\end{equation}
We can go further and show that we can talk about $\bm{\tau}_2$ without solving for $\bm{z}_1^*$. Consider Equation \ref{eq:zstar} once again,

\begin{equation}
    | \bm{c}_i^T \bm{z}_1 | \leq \| \bm{c}_i \|_2 \| \bm{z}_1 \|_2 = \| \bm{c}_i \|_2
\end{equation}
Hence, $z_{2i} = 0$ for $i \in 1, \ldots, p_2$ if $\| \bm{c}_i \|_2 \leq \gamma_2$ without regard to $\bm{z}_1^*$. 
\end{proof}

Program \ref{eq:z1} can be viewed as a $L_1$ regularized maximization of a quadratic function over a compact set. Obviously the objective is not convex, since it's the difference of two convex functions. However, as we will elaborate more Chapter \ref{sec:mule} where we propose our two-stage algorithm, \texttt{MuLe}, we are only interested in $\bm{z}_1^*$ for the purpose of inferring $\bm{\tau}_2$. Hence we will optimize Program \ref{eq:z1l1} with no regularization term in the first stage.

\begin{align}
\label{eq:shrink}
\phi_{l_1,l_1}^2(\gamma_1, \gamma_2) \approx \max_{\bm{z}_1 \in \mathcal{B}^{p_1}} \sum_{i= 1}^{p_2} [ | \bm{c}_i^T \bm{z}_1 | - \gamma_2  ]_+^2
= \max_{\bm{z}_1 \in \mathcal{S}^{p_1}} \sum_{i= 1}^{p_2} [ | \bm{c}_i^T \bm{z}_1 | - \gamma_2  ]_+^2
\end{align}

\begin{remark}
\label{rmk:shrinkl1}
As a result of this approximation, as stated in Program \ref{eq:shrink}, the search space is drastically shrunk from a $p_1$-dimensional Euclidean ball to a $p_1$-dimensional sphere. This is as a result of maximizing a convex function over a compact set.
\end{remark}

\begin{remark}
\label{rmk:approxl1}
Program \ref{eq:shrink} is a valid approximation of the Program \ref{eq:z1l1}. Beside our simulation results in Section \ref{sec:sim}, we can see that there is a one-to-one mapping $\gamma_1 = h(\gamma_2)$ in light of Equation \ref{eq:sparsityl1}; in other words, for every $\gamma_1$ for which $z_{1i}^* = 0$ there is a $\gamma_2$ for which the last inequality in \ref{eq:sparsityl1} is true.
\end{remark}

\subsection{\texorpdfstring{$L_0$}{TEXT} Regularization}
\label{sec:l0reg}
Adapting formulation \ref{eq:sCCAl0} of \cite{wiesel2008greedy} to our approach is equivalent to setting $x = 0$ in \ref{eq:ref},

\begin{equation}
\label{eq:l0reg}
\phi_{l_0, l_0}(\gamma_1, \gamma_2) = \max_{\bm{z}_1 \in \mathcal{B}^{p_1} } \max_{ \bm{z}_2 \in \mathcal{B}^{p_2} } \bm{z}_1^T\bm{C}_{12}\bm{z}_2 - \gamma_1 \| \bm{z}_1 \|_0 - \gamma_2 \| \bm{z}_2 \|_0 
\end{equation}

However, to make use of the results in the previous section, we consider the following program instead,

\begin{equation}
\label{eq:l0regref}
\phi_{l_0, l_0}'(\gamma_1, \gamma_2) = \max_{\bm{z}_1 \in \mathcal{B}^{p_1} } \max_{ \bm{z}_2 \in \mathcal{B}^{p_2} } (\bm{z}_1^T\bm{C}_{12}\bm{z}_2)^2 - \gamma_1 \| \bm{z}_1 \|_0 - \gamma_2 \| \bm{z}_2 \|_0 
\end{equation}


\begin{theorem}
\label{thm:l0}
Maximizers, $(\bm{z}_1^*, \bm{z}_2^*)$, to $\phi_{l_0, l_0}(\gamma_1, \gamma_2)$ in Program \ref{eq:l0reg} are given by,

\begin{equation}
\label{eq:z1l0}
\bm{z}_1^* = \argmax_{\bm{z}_1 \in \mathcal{B}^{p_1}} \sum_{i= 1}^{p_2} [ ( \bm{c}_i^T \bm{z}_1 )^2 - \gamma_2  ]_+ - \gamma_1 \|\bm{z}_1\|_0
\end{equation}

and

\begin{equation}
\label{eq:zstarl0}
z_{2i}^* = z_{2i}^*(\gamma_2) = \frac{[ sgn( (\bm{c}_i^T \bm{z}_1 )^2 - \gamma_2)  ]_+ \bm{c}_i^{\top}\bm{z}_1}{\sqrt{\sum_{k=1}^{p_2} [ sgn( (\bm{c}_k^T \bm{z}_1 )^2 - \gamma_2)  ]_+ (\bm{c}_k^{\top}\bm{z}_1)^2 } }, \quad i = 1, \ldots, p_2.
\end{equation}

\end{theorem}

\begin{proof}
Consider optimizing over $\bm{z}_2$ while keeping $\bm{z}_1$ fixed. First, assume $\gamma_2 = 0$. Obviously, $\phi_{l_0, l_0}(\gamma_1, 0)|_{\bm{z}_1 = const.}$ is maximized at $\bm{z}_2^* = \bm{c}_i^{\top}\bm{z}_1$. Now, considering the case for $\gamma_2 > 0$, for which $z_{2i}^* = 0$ for any $\bm{z}_1$ such that $\phi_{l_0, l_0}(\gamma_1, 0)|_{\bm{z}_1 = const.} = (\bm{c}_i^T \bm{z}_1 )^2 \leq \gamma_2$. Considering this analysis and normalizing we obtain Equation \ref{eq:zstarl0}. Substituting back in \ref{eq:l0regref}, we arrive at \ref{eq:z1l0}.

\end{proof}

Similar to the $L_1$ regularized case, the following corollary formalizes the relationship between $\bm{z}_1^*$ and the sparsity pattern $\bm{\tau}_2 \in \{0,1\}^{p_2}$ of $\bm{z}_2^*$.

\begin{corollary}
\label{cor:sparsityl0}
Given the sparsity parameter $\gamma_2$ and solution $\bm{z}_1^*$ to the program \ref{eq:z1l0}, 

\begin{equation}
 \label{eq:spl0}
\bm{\tau}_{2i} = \begin{cases}
0 & -\sqrt{\gamma_2} \leq \bm{c}_i^{\top}\bm{z}_1^* \leq \sqrt{\gamma_2}\\
1 & otherwise
\end{cases}
\end{equation}

\end{corollary}

\begin{proof}
According to Equation \ref{eq:zstarl0} of Theorem \ref{thm:l0}, 

\begin{equation}
\label{eq:sparsityl0}
  z_{2i}^* = 0 \Leftrightarrow  sgn( (\bm{c}_i^T \bm{z}_1^* )^2 - \gamma_2) \leq 0 \Leftrightarrow (\bm{c}_i^T \bm{z}_1^* )^2 \leq \gamma_2
\end{equation}
Again, even without solving for $\bm{z}_1^*$ we can show that

\begin{equation}
    ( \bm{c}_i^T \bm{z}_1 )^2 \leq \| \bm{c}_i \|_2^2 \| \bm{z}_1 \|_2^2 = \| \bm{c}_i \|_2^2
\end{equation}
Hence, in light of \ref{eq:zstarl0}, $z_{2i} = 0$ for $i \in 1, \ldots, p_2$ if $\| \bm{c}_i \|_2^2 \leq \gamma_2$ without regards to $\bm{z}_1^*$. 
\end{proof}

As before, Program \ref{eq:z1l0} can be viewed as a $L_0$ regularized maximization of a quadratic function over a compact set. Also, we are only interested in $\bm{z}_1^*$ for the purpose of inferring $\bm{\tau}_2$. Therefore, to be able to use the previous result in shrinking the search domain, we will optimize Program \ref{eq:z1l0} with no regularization in the first stage.

\begin{align}
\label{eq:shrinkl0}
\phi_{l_0,l_0}'(\gamma_1, \gamma_2) \approx \max_{\bm{z}_1 \in \mathcal{B}^{p_1}} \sum_{i= 1}^{p_2} [ ( \bm{c}_i^{\top} \bm{z}_1 )^2 - \gamma_2  ]_+
= \max_{\bm{z}_1 \in \mathcal{S}^{p_1}} \sum_{i= 1}^{p_2} [ ( \bm{c}_i^{\top} \bm{z}_1 )^2 - \gamma_2  ]_+
\end{align}

The same justifications as presented in Remarks \ref{rmk:shrinkl1} and \ref{rmk:approxl1} apply here analogously.

\bigskip

\noindent So far we proposed methods to infer the sparsity patterns $\bm{\tau}_1$ and $\bm{\tau}_2$ which can be used to shrink the covariance matrix drastically, as explain in Section \ref{sec:mule}. Now, efficient CCA algorithms may be used to estimate the active entries of $\bm{z}_1^*$ and $\bm{z}_2^*$. Assuming we have estimated the $i-th$ pair of canonical loading vectors, $(\bm{z}_1, \bm{z}_2)^{(i)}, i = 1, \ldots, I$, where $I = rank(\bm{C}_{12}) \leq n$ assuming $n << min\{p_1, p_2\}$, we define the \textit{i-th Residual Covariance Matrix} as,

\begin{equation}
\bm{C}_{12}^{(i)} = \bm{C}_{12} - \sum_{k = 1}^{i} (\bm{z}_1^{(k)*\top}\bm{C}_{12}^{(k-1)}\bm{z}_2^{(k)*}) \bm{z}_1^{(k)*}\bm{z}_2^{(k)*\top} \quad 1 \leq i \leq I
\end{equation}

The $(i+1)-th$ pair of canonical loading vectors are estimated by the leading canonical loading vectors of $\bm{C}_{12}^{(i)}$, using any of the previous two methods. Refer to Algorithm \ref{alg:deflation} in Appendix \ref{app:alg:multifactor} for more details.








\section{Further Applications and Extensions}
\label{sec:ext}

In this section we further extend the methods developed in Section \ref{sec:convex}.
In \ref{subsec:multimodal} we introduce our approach to \textit{Multi-View Sparse CCA}, where more than two views are available. In \ref{subsec:directed} we extend our approach to \textit{Directed Sparse CCA}, where an observed variable, other than the observed views, is available, towards which we direct the canonical directions.



\subsection{Multi-View Sparse CCA}
\label{subsec:multimodal}

So far we limited ourselves to a pair of views in discussing the sub-space learning problem. In this section we extend our approach to learning sub-spaces from multiple views, i.e. when we have multiple groups of observations, $\bm{X}_i \in \mathbb{R}^{n \times p_i}, i = 1, \ldots, m$ on matching samples. An example of this problem is multi-omic genetic studies where transcriptomic, metabolomic, and microbiomic data are collected from a single group of individuals. Thus, we try to discover the association structures between random vectors $X_i$ by estimating $\bm{z}_i$ such that $\bm{X}_i\bm{z}_i$ are maximally correlated in pairs. Here, we propose a solution to the following optimization program which is equivalent to the one proposed in \cite{witten:tibshirani:2009},

\begin{equation}
\label{eq:mCCA}
\phi_{l_x}^M(\bm{\Gamma}) = \max_{\substack{\bm{z}_i \in \mathcal{B}^{p_i}\\ \forall i = 1, \ldots, m }} \sum_{r<s = 2}^{m} \bm{z}_r^T\bm{C}_{rs}\bm{z}_s - \sum_{s = 2}^m  \sum_{\substack{r = 1 \\ r \neq s }}^{s-1} \Gamma_{sr} \| \bm{z}_s \|_1 
\end{equation}

where $m$ is the total number of available views, $\bm{\Gamma} \in \mathbb{R}^{m \times m}$, $\Gamma_{ij} \geq 0$ is a Lagrange multiplier matrix, and $\bm{C}_{rs} = 1/n \bm{X}_r^T\bm{X}_s$ is the sample covariance matrix of the $(r,s)$ pairs of views. Following similar procedure as in \ref{sec:l1reg}, we analyze the solution to Program \ref{eq:mCCA}.


\begin{theorem}
\label{thm:mccal1}
The local optima $\bm{z}_1^*, \ldots, \bm{z}_m^*$ of the optimization problem \ref{eq:mCCA} is given by,

\begin{equation}
\label{eq:multizstar}
z_{si}^* = z_{si}^*(\bm{\Gamma}) = \frac{sgn(\sum_{ \substack{r = 1\\ r \neq s}}^m  \tilde{\bm{c}}_{rsi}^{\top} \bm{z}_r) [|\sum_{ \substack{r = 1\\ r \neq s}}^m  \tilde{\bm{c}}_{rsi}^{\top} \bm{z}_r| - \sum_{\substack{r = 1\\ r \neq s}}^{m} \Gamma_{sr}]_+}{\sqrt{\sum_{k=1}^{p_2} [|\sum_{ \substack{r = 1\\ r \neq s}}^m  \tilde{\bm{c}}_{rsk}^{\top} \bm{z}_r| - \sum_{\substack{r = 1\\ r \neq s}}^{m} \Gamma_{sr}]_+^2 } }
\end{equation}

and for $r = 1, \ldots, m$ and $r \neq s$,

\begin{equation}
\label{eq:multizn}
\begin{split}
    \bm{z}_r(\bm{\Gamma}) = \max_{\substack{\bm{z}_r \in \mathcal{B}^{p_r}\\ r \neq s, r = 1, \ldots, m} } &\sum_{i= 1}^{p_s} [|\sum_{ \substack{r = 1\\ r \neq s}}^m  \tilde{\bm{c}}_{rsi}^{\top} \bm{z}_r| - \sum_{\substack{r = 1\\ r \neq s}}^{m} \Gamma_{sr}]_+^2 +\\  
    &\sum_{ \substack{i < j = 2\\ i, j \neq s }}^m \bm{z}_i^{\top}\bm{C}_{ij}\bm{z}_j  -\sum_{\substack{i = 1 \\ i \neq s} }^m  \sum_{\substack{j = 1 \\ i \neq j }}^{m-1} \Gamma_{ij} \| \bm{z}_i \|_1
\end{split}
\end{equation}

\end{theorem}



\begin{proof}
Here we follow a progression similar to the  proof of Theorem \ref{thm:l1}.

\begin{alignat}{2}
\label{eq:mccaformulation}
\phi_{l_1}^{m}(\bm{\Gamma}) &= \max_{\substack{\bm{z}_r \in \mathcal{B}^{p_r}\\ r \neq s, r = 1, \ldots, m} } \max_{ \bm{z}_s \in \mathcal{B}^{p_s} }&&\sum_{ r < s = 2}^m \bm{z}_r^{\top}\bm{C}_{rs}\bm{z}_s - \sum_{s = 1}^m  \sum_{\substack{r = 1 \\ r \neq s }}^{m-1} \Gamma_{sr} \| \bm{z}_s \|_1\\
\nonumber
&= \max_{\substack{\bm{z}_r \in \mathcal{B}^{p_r}\\ r \neq s, r = 1, \ldots, m} } \max_{ \bm{z}_s \in \mathcal{B}^{p_s} } && \sum_{i = 1}^{p_s} z_{si}(\sum_{ \substack{r = 1\\ r \neq s}}^m  \tilde{\bm{c}}_{rsi}^{\top} \bm{z}_r) - \sum_{\substack{r = 1\\ r \neq s}}^{m} \Gamma_{sr} \| \bm{z}_s \|_1 +\\
& &&\overbrace{\sum_{ \substack{i < j = 2\\ i, j \neq s }}^m \bm{z}_i^{\top}\bm{C}_{ij}\bm{z}_j  -\sum_{\substack{i = 1 \\ i \neq s} }^m  \sum_{\substack{j = 1 \\ i \neq j }}^{i-1} \Gamma_{ij} \| \bm{z}_i \|_1}^\text{\textit{I}}\\
\label{eq:mccaformulationlastline}
&= \max_{\substack{\bm{z}_r \in \mathcal{B}^{p_r}\\ r \neq s, r = 1, \ldots, m} } \max_{ \bm{z}_s \in \mathcal{B}^{p_s} } && \sum_{i = 1}^{p_s} |z_{si}'|(|\sum_{ \substack{r = 1\\ r \neq s}}^m  \tilde{\bm{c}}_{rsi}^{\top} \bm{z}_r| - \sum_{\substack{r = 1\\ r \neq s}}^{m} \Gamma_{sr}) + I
\end{alignat}

The last line follows from $z_{si} = sgn(\sum_{ \substack{r = 1\\ r \neq s}}^m  \tilde{\bm{c}}_{rsi}^{\top} \bm{z}_r)z_{si}'$. $\tilde{\bm{c}}_{rsi} = \bm{c}_{rsi}$ if $r< s$, and $\tilde{\bm{c}}_{rsi} = \bm{c}_{rsi}^{\top}$ if $r > s$ where $\bm{c}_{rsi}$ is the $i$th row of $\bm{C}_{rs} = 1/n \bm{X}_r^T \bm{X}_s$. Solving for $\bm{z}_s'$ and converting back to $\bm{z}_s$, using the aforementioned change-of-variable and normalizing, we get the local optimum in \ref{eq:multizstar}. Substituting back to \ref{eq:mccaformulationlastline},

\begin{equation}
\label{eq:multiznopt}
\begin{split}
    {\phi_{l_1}^m}^2(\bm{\Gamma}) = \max_{\substack{\bm{z}_r \in \mathcal{B}^{p_r}\\ r \neq s, r = 1, \ldots, m} } &\sum_{i= 1}^{p_s} [|\sum_{ \substack{r = 1\\ r \neq s}}^m  \tilde{\bm{c}}_{rsk}^{\top} \bm{z}_r| - \sum_{\substack{r = 1\\ r \neq s}}^{m} \Gamma_{sr}]_+^2 +\\  
    &\sum_{ \substack{i < j = 2\\ i, j \neq s }}^m \bm{z}_i^{\top}\bm{C}_{ij}\bm{z}_j  -\sum_{\substack{i = 1 \\ i \neq s} }^m  \sum_{\substack{j = 1 \\ i \neq j }}^{m-1} \Gamma_{ij} \| \bm{z}_i \|_1
\end{split}
\end{equation}

\end{proof}


As pointed out in Section \ref{sec:l1reg}, we're only interested in the optimizing \ref{eq:multiznopt} in order to find the sparsity pattern $\bm{\tau}_s \in \{0,1\}^{p_s}$. Per Remark \ref{rmk:approxl1}, we can make a good approximation by not considering the regularization terms, simplifying the problem to,

\begin{equation}
\label{eq:multiznoptsimple}
\begin{split}
    {\phi_{l_1}^m}^2(\bm{\Gamma}) = \max_{\substack{\bm{z}_r \in \mathcal{B}^{p_r}\\ r \neq s, r = 1, \ldots, m} } \sum_{i= 1}^{p_s} [|\sum_{ \substack{r = 1\\ r \neq s}}^m  \tilde{\bm{c}}_{rsi}^{\top} \bm{z}_r| - \sum_{\substack{r = 1\\ r \neq s}}^{m} \Gamma_{sr}]_+^2 + \sum_{ \substack{i < j = 2\\ i, j \neq s }}^m \bm{z}_i^{\top}\bm{C}_{ij}\bm{z}_j
\end{split}
\end{equation}

As before, we can talk about $\bm{\tau}_s$, by just looking at $\bm{z}_{r}^*$ for $r = 1, \ldots, m$ and $r \neq s$.

\begin{corollary}
\label{cor:mccasparsityl0}
For a sparsity parameter matrix $\bm{\Gamma}$ and the solution, $\bm{z}_r^*$ for $r = 1, \ldots, m$ and $r \neq s$, to the Program \ref{eq:multiznoptsimple},

\begin{equation}
 \label{eq:mccaspl0}
\bm{\tau}_{2i} = \begin{cases}
0 & |\sum_{ \substack{r = 1\\ r \neq s}}^m  \tilde{\bm{c}}_{rsi}^{\top} \bm{z}_r| \leq \sum_{\substack{r = 1\\ r \neq s}}^{m} \Gamma_{sr}\\
1 & otherwise
\end{cases}
\end{equation}

\end{corollary}

\begin{proof}
Scanning Equation \ref{eq:multizstar},

\begin{equation}
\label{eq:multisparsityl0}
  z_{si}^* = 0 \Leftrightarrow  [|\sum_{ \substack{r = 1\\ r \neq s}}^m \tilde{\bm{c}}_{rsi}^{\top} \bm{z}_r| - \sum_{\substack{r = 1\\ r \neq s}}^{m} \Gamma_{sr}]_+^2 = 0 \Leftrightarrow |\sum_{ \substack{r = 1\\ r \neq s}}^m  \tilde{\bm{c}}_{rsi}^{\top} \bm{z}_r| \leq \sum_{\substack{r = 1\\ r \neq s}}^{m} \Gamma_{sr}
\end{equation}

Regardless of $\bm{z}_r^*$ we have,

\begin{equation}
   |\sum_{ \substack{r = 1\\ r \neq s}}^m \tilde{\bm{c}}_{rsi}^{\top} \bm{z}_r| \leq \sum_{ \substack{r = 1\\ r \neq s}}^m \| \tilde{\bm{c}}_{rsi} \|_2 \| \bm{z}_r \|_2 = \sum_{ \substack{r = 1\\ r \neq s}}^m \| \tilde{\bm{c}}_{rsi} \|_2
\end{equation}

Hence, $\tau_{si} = 0$ for $i \in 1, \ldots, p_s$ if $\sum_{ \substack{r = 1\\ r \neq s}}^m \| \tilde{\bm{c}}_{rsi} \|_2 \leq \sum_{\substack{r = 1\\ r \neq s}}^{m} \Gamma_{sr} $ regardless of $\bm{z}_r^*$. 
\end{proof}

Computing $\bm{\tau}_i$ is the first stage of our two-stage multi-modal sCCA approach, for which a fast algorithm is proposed in \ref{subsec:multimodalalg} as part of our proposed \texttt{MuLe} framework. The second stage of our approach consists of estimating the active elements of $\bm{z}_i^*$, for which we use two methods, one is to frame the multi-modal CCA problem as a generalized eigenvalue problem as originally proposed in \cite{kettenring1971canonical}, see Appendix \ref{app:mcca}, and the other one is a more algorithmic approach of extending SVD via power iterations to multiple views, refer to Appendix \ref{app:msvd}.

\subsection{Directed Sparse CCA}
\label{subsec:directed}

Consider a setting where in addition to the views $\bm{X}_i \in \mathbb{R}^{n \times p_i}$, some \textit{accessory variable}\footnote{We coined the term \textit{Accessory Variable} to prevent confusion about the causal role of $\bm{y}$, and to emphasize that independent from their role, whether dependent or independent variable, we are solely utilizing them as a direction towards which we're directing the canonical directions.}, $Y(\omega): \omega \rightarrow \mathbb{R}$ $\bm{y} \in \mathbb{R}^{n}$, is also observed. We also term the observed accessory variable the \textit{Accessory Direction}, $\bm{y} \in \mathbb{R}^{n}$. Having observed $\bm{y}$, the objective is to find linear combinations of the covariates in each view which are highly correlated with each other and also ``associated" with the accessory direction. This is useful in high-dimensional settings where rank-deficient covariance matrices lead to over-fitting, and small sample sizes are not representative of the direction of variance within each population, and particularly useful in hypotheses generation where we're interested in correlation structures associated with a specific experiment design, e.g. association mechanisms corresponding to a certain treatment effect. Here we compare two approaches to this problem,

\subsubsection{Two-Step Formulation}
\label{subsubsec:twostep}

\cite{witten:tibshirani:2009} propose \textit{Sparse Supervised CCA}, where they consider an extra observed outcome. Their approach consists of two sequential steps where the first step, which is completely separate from the second step, involves finding subsets $Q_i$ of each random vector $X_i$ using a conventional variable selection method, e.g. LASSO regression. In the second step, they utilize sparse CCA where the scope of search and estimation of the canonical directions is limited to the subspaces defined by $X_{ij}, j \in Q_i$,

\begin{equation}
\label{eq:l1directed}
\phi_{l_1, l_1}(\gamma_1, \gamma_2) = \max_{\substack{\bm{z}_1 \in \mathcal{B}^{p_1}\\ z_{1j} = 0, \forall j \in Q_1} } \max_{ \substack{\bm{z}_2 \in \mathcal{B}^{p_2}\\ z_{2j} = 0, \forall j \in Q_2} } \bm{z}_1^T\bm{C}_{12}\bm{z}_2 - \gamma_1 \| \bm{z}_1 \|_1 - \gamma_2 \| \bm{z}_2 \|_1 
\end{equation}

In Appendix \ref{app:semimule} a simple algorithm to optimize \ref{eq:l1directed} is introduced. This approach, however, has two considerable shortcomings:

\begin{enumerate}
    \item Although the scopes of canonical directions are limited to the subspace spanned by $\bm{z}_i \in \mathcal{B}^{p_i}, z_{ij} = 0, \forall j \in Q_i$, the active elements of these directions are estimated to maximize the sCCA criterion. The estimated direction may well not be associated to the outcome vector anymore, which misses the point.
    \item Computing $Q_i$ requires some parameter tuning, e.g. sparsity parameters, which is blind to the CCA criterion; as a result, $Q_i$ might exclude covariates which are moderately correlated with $\bm{y}$ but highly associated with covariates in other views. 
\end{enumerate}

To bridge the gap between the two stages, we propose an approach where $\bm{z}_i$ are estimated in one stage such that the canonical covariates are highly correlated with each other and also associated with the accessory variable.

\subsubsection{Single-Stage Formulation}
\label{subsubsec:singlestage}

The following optimization problem tends to perform the two stages of variable selection and performing sCCA in one stage simultaneously,

\begin{equation}
\label{eq:singlestage}
\phi_{l_1, l_1}^D(\bm{\gamma}, \bm{\epsilon}) = \max_{\bm{z}_1 \in \mathcal{B}^{p_1} } \max_{ \bm{z}_2 \in \mathcal{B}^{p_2}} \bm{z}_1^T\bm{C}_{12}\bm{z}_2 - \sum_{i = 1}^2 [\epsilon_i \mathcal{L}_i(\bm{X}_i\bm{z}_i, \bm{y}) + \gamma_i \| \bm{z}_i \|_1] 
\end{equation}

where $\mathcal{L}_i$ is some loss function which directs our canonical directions to be associated with the accessory direction $\bm{y}$, and $\gamma_i, \epsilon_i \in \mathbb{R}, i = 1,2$ are non-negative Lagrange multipliers. Here we analyze two scenarios,

\textbf{a}. Let's consider the case where $\bm{y}$ is another separate explanatory variable. Here, one possible utility function is the dot-product between the canonical covariates and the explanatory variable, i.e. $\mathcal{L}(\bm{X}_i\bm{z}_i, \bm{y}) = - \langle \bm{X}_i\bm{z}_i, \bm{y} \rangle$. Replacing in \ref{eq:singlestage}, we have,

\begin{equation}
\label{eq:singlestagecor}
\phi_{l_1, l_1}^D(\bm{\gamma}, \bm{\epsilon}) = \max_{\bm{z}_1 \in \mathcal{B}^{p_1} } \max_{ \bm{z}_2 \in \mathcal{B}^{p_2}} \bm{z}_1^T\bm{C}_{12}\bm{z}_2 + \sum_{i = 1}^2 [\epsilon_i \bm{y}^{\top}\bm{X}_i\bm{z}_i - \gamma_i \| \bm{z}_i \|_1] 
\end{equation}


\begin{theorem}
\label{thm:singlestageinnerprod}
The local optima, $(\bm{z}_1^*, \bm{z}_2^*)$, to $\phi_{l_1, l_1}^D(\bm{\gamma}, \bm{\epsilon})$ in optimization program \ref{eq:singlestagecor} is given by,

\begin{equation}
\label{eq:directedcorz1}
\bm{z}_1^* = \argmax_{\bm{z}_1 \in \mathcal{B}^{p_1}} \sum_{i= 1}^{p_2} [ | \bm{c}_i^T\bm{z}_1 + \epsilon_2 \bm{x}_{2i}^{\top}\bm{y} | - \gamma_2  ]_+^2 + \epsilon_1 \bm{y}\bm{X}_1\bm{z}_1 - \gamma_1 \|\bm{z_1}\|_1
\end{equation}

and

\begin{equation}
\label{eq:zstardirectedcor}
z_{2i}^* = z_{2i}^*(\gamma_2, \epsilon_2) = \frac{sgn( \bm{c}_i^T\bm{z}_1 + \epsilon_2 \bm{x}_{2i}^{\top}\bm{y} ) [ | \bm{c}_i^T\bm{z}_1 + \epsilon_2 \bm{x}_{2i}^{\top}\bm{y} | - \gamma_2  ]_+}{\sqrt{\sum_{k=1}^{p_2} [|\bm{c}_i^T\bm{z}_1 + \epsilon_2 \bm{x}_{2i}^{\top}\bm{y} | - \gamma_2]_+^2 } }, \quad i = 1, \ldots, p_2.
\end{equation}
\end{theorem}

\begin{proof}

\begin{align}
\label{eq:directedl1reformulation}
\begin{split}
\phi_{l_1, l_1}^D(\bm{\gamma}, \bm{\epsilon}) &= \max_{\bm{z}_1 \in \mathcal{B}^{p_1} } \max_{ \bm{z}_2 \in \mathcal{B}^{p_2}} \bm{z}_1^T\bm{C}_{12}\bm{z}_2 + \sum_{i = 1}^2 [\epsilon_i \bm{y}^{\top}\bm{X}_i\bm{z}_i - \gamma_i \| \bm{z}_i \|_1] \\
&= \max_{\bm{z}_1 \in \mathcal{B}^{p_1} } \max_{ \bm{z}_2 \in \mathcal{B}^{p_2}} \sum_{i = 1}^{p_2} z_{2i}(\bm{c}_i^T \bm{z}_1 + \epsilon_2 \bm{x}_{2i}^{\top}\bm{y}) - \gamma_2\| \bm{z}_2 \|_1 + \epsilon_1 \bm{y}\bm{X}_1\bm{z}_1 - \gamma_1 \| \bm{z}_1 \|_1\\
&= \max_{\bm{z}_1 \in \mathcal{B}^{p_1} } \max_{ \bm{z}_2 \in \mathcal{B}^{p_2}} \sum_{i = 1}^{p_2} |z_{2i}'|(|\bm{c}_i^T \bm{z}_1 + \epsilon_2 \bm{x}_{2i}^{\top}\bm{y}| - \gamma_2) + \epsilon_1 \bm{y}\bm{X}_1\bm{z}_1 - \gamma_1 \| \bm{z}_1 \|_1
\end{split}
\end{align}

As before we used a simple change of variable, $\bm{z}_{2i} = sgn(\bm{c}_i^T\bm{z}_1 + \epsilon_2 \bm{x}_i^{\top}\bm{y})\bm{z}_{2i}'$. We solve \ref{eq:directedl1reformulation} for $\bm{z}_2'$ for fixed $\bm{z}_1$ and convert it back, using the aformentioned change-of-variable, to $\bm{z}_2$ to get the result in Equation \ref{eq:zstardirectedcor}. Substituting this result back in \ref{eq:directedl1reformulation},

\begin{equation}
\label{eq:directedz1l1}
{\phi_{l_1,l_1}^{D}}^2(\bm{\gamma}, \bm{\epsilon}) = \max_{\bm{z}_1 \in \mathcal{B}^{p_1}} \sum_{i= 1}^{p_2} [ | \bm{c}_i^T\bm{z}_1 + \epsilon_2 \bm{x}_{2i}^{\top}\bm{y} | - \gamma_2  ]_+^2 + \epsilon_1 \bm{y}\bm{X}_1\bm{z}_1 - \gamma_1 \|\bm{z_1}\|_1
\end{equation}
\end{proof}

Quite similar to our sCCA formulation we can find the sparsity pattern, $\bm{\tau}_2$ of $\bm{z}_2^*$ by looking at $\bm{z}_1^*$.

\begin{corollary}
\label{cor:sparsitydirectedcor}
Given hyperparameters $\gamma_2$, $\epsilon_2$, and $\bm{z}_1^*$ from program \ref{eq:directedcorz1}, $\tau_{2i} = 0$ if $| \bm{c}_i^T\bm{z}_1^* + \epsilon_2 \bm{x}_{2i}^{\top}\bm{y} | \leq \gamma_2$.
\end{corollary}

\begin{proof}
According to Equation \ref{eq:zstardirectedcor} of Theorem \ref{thm:singlestageinnerprod}, 

\begin{equation}
\label{eq:directedcorsparsityl1}
  z_{2i}^* = 0 \Leftrightarrow [ | \bm{c}_i^T\bm{z}_1^* + \epsilon_2 \bm{x}_{2i}^{\top}\bm{y} | - \gamma_2  ]_+  = 0 \Leftrightarrow |\bm{c}_i^T\bm{z}_1^* + \epsilon_2 \bm{x}_{2i}^{\top}\bm{y} | \leq \gamma_2
\end{equation}
We can go further and show that we can talk about $\bm{\tau}_2$ without solving for $\bm{z}_1^*$,

\begin{equation}
    | \bm{c}_i^T\bm{z}_1 + \epsilon_2 \bm{x}_{2i}^{\top}\bm{y} | \leq \| \bm{c}_i \|_2 \| \bm{z}_1 \|_2 + \epsilon_2 \| \bm{x}_{2i}\|_2 \|\bm{y}\|_2= \| \bm{c}_i \|_2 + \epsilon_2 \| \bm{x}_{2i} \|_2
\end{equation}
Hence, $z_{2i} = 0$ for $i \in 1, \ldots, p_2$ if $\| \bm{c}_i \|_2 + \epsilon_2 \| \bm{x}_{2i} \|_2 \leq \gamma_2$ regardless of $\bm{z}_1^*$. 
\end{proof}

\textbf{b}. Let's examine a setting where $\bm{y}$ is an outcome variable. Here the objective is to ideally find a common low-dimensional subspace in which the projections of $\bm{X}_i$ are as correlated as possible and also descriptive/predictive of the outcome $\bm{y}$. Being confined to linear projections, we can choose $\mathcal{L}_i(\bm{X}_i\bm{z}_i, \bm{y}) = \| \bm{y} - \bm{X}_i\bm{z}_i\|_2^2$, i.e. sum of squared errors loss. Rewriting \ref{eq:singlestage} with this choice,

\begin{equation}
\label{eq:singlestagereg}
\phi_{l_1, l_1}^D(\bm{\gamma}, \bm{\epsilon}) = \max_{\bm{z}_1 \in \mathcal{B}^{p_1} } \max_{ \bm{z}_2 \in \mathcal{B}^{p_2}} \bm{z}_1^T\bm{C}_{12}\bm{z}_2 - \sum_{i = 1}^2 [\epsilon_i  \| \bm{y} - \bm{X}_i\bm{z}_i\|_2^2 + \gamma_i \| \bm{z}_i \|_1] 
\end{equation}


\begin{theorem}
\label{thm:singlestagereg}
The optimization program in \ref{eq:singlestagereg} is equivalent to the following program,

\begin{equation}
\label{eq:singlestageregequiv}
\phi_{l_1, l_1}^D(\bm{\gamma}, \bm{\epsilon}) = \max_{\bm{z} \in \mathcal{B}^{p} } \bm{z}^{\top}\tilde{\bm{C}}\bm{z} + 2 \bm{y}^{\top}\tilde{\bm{X}}\bm{z} - \gamma_1 \| \bm{z}_1 \|_1 - \gamma_1 \| \bm{z}_2 \|_1 
\end{equation}

where,

\begin{equation}
    \tilde{\bm{z}} = \begin{bmatrix}
    \bm{z}_1 \\ \bm{z}_2
    \end{bmatrix}, \quad
    \tilde{\bm{C}} = \begin{bmatrix}
    \epsilon_1 \bm{C}_{11} & \bm{C}_{12} \\ \bm{C}_{12}^{\top} & \epsilon_2 \bm{C}_{22} \end{bmatrix}, \quad \tilde{\bm{X}} = \begin{bmatrix}
    \epsilon_1 \bm{X}_1 & \epsilon_2 \bm{X}_2
    \end{bmatrix}, \quad
\end{equation}

and $p = p_1 + p_2$. The solution, $(\bm{z}_1^*, \bm{z}_2^*)$, to $\phi_{l_1, l_1}^D(\bm{\gamma}, \bm{\epsilon})$ in Program \ref{eq:singlestagereg} is given by,

\begin{equation}
\label{eq:directedregz1}
\bm{v}^* = \argmax_{\bm{v} \in \mathcal{B}^{p}} \sum_{i= 1}^{p_2} [ | \tilde{\bm{c}}_i^T\bm{v} + 2 \tilde{\bm{x}}_{i}^{\top}\bm{y} | - \gamma_1 I_{( i \leq p_1  )} - \gamma_2 I_{(p_1 < i)} ]_+^2
\end{equation}

and

\begin{equation}
\label{eq:zstardirectedreg}
z_{i}^* = z_{i}^*(\bm{\gamma}, \bm{\epsilon}) = \frac{sgn( \tilde{\bm{c}}_i^{\top}\bm{v} + 2 \tilde{\bm{x}}_{i}^{\top}\bm{y} ) [ | \tilde{\bm{c}}_i^T\bm{v} + 2 \tilde{\bm{x}}_{i}^{\top}\bm{y} | - \gamma_1 I_{( i \leq p_1  )} - \gamma_2 I_{(p_1 < i)} ]_+}{\sqrt{\sum_{k=1}^{p} [| \tilde{\bm{c}}_k^T\bm{v} + 2 \tilde{\bm{x}}_{k}^{\top}\bm{y} | - \gamma_1 I_{( k \leq p_1  )} - \gamma_2 I_{(p_1 < k)} ]_+^2 } }, \quad i = 1, \ldots, p_2.
\end{equation}
\end{theorem}

\begin{proof}
Let $\bm{R} = \tilde{\bm{C}}^{1/2}$.
\begin{align}
\label{eq:directedl1reformulationreg}
\begin{split}
\phi_{l_1, l_1}^D(\bm{\gamma}, \bm{\epsilon}) &= \max_{\bm{z} \in \mathcal{B}^{p} } \max_{\bm{v} \in \mathcal{B}^{p}} \bm{v}^{\top}\tilde{\bm{C}}^{1/2}\bm{z} + 2 \bm{y}^{\top}\tilde{\bm{X}}\bm{z} - \gamma_1 \| \bm{z}_1 \|_1 - \gamma_2 \| \bm{z}_2 \|_1 \\
&=  \max_{\bm{v} \in \mathcal{B}^{p}} \max_{\bm{z} \in \mathcal{B}^{p} } \sum_{i =1}^p z_{i}(\tilde{\bm{c}}_i^{\top}\bm{v} + 2\tilde{\bm{x}}_i^{\top} \bm{y}) - \gamma_2 \| \bm{z}_2 \|_1 - \gamma_1 \| \bm{z}_1 \|_1\\
&= \max_{\bm{v} \in \mathcal{B}^{p}} \max_{\bm{z} \in \mathcal{B}^{p} } \sum_{i =1}^p |z_{i}'|(|\tilde{\bm{c}}_i^{\top}\bm{v} + 2\tilde{\bm{x}}_i^{\top} \bm{y}| - \gamma_1 I_{( i \leq p_1  )} - \gamma_2 I_{(p_1 < i)}) \\
\end{split}
\end{align}

where $z_{i} = sgn(\tilde{\bm{c}}_i^{\top}\bm{v} + 2\tilde{\bm{x}}_i^{\top} \bm{y})z_i'$. We optimize \ref{eq:directedl1reformulationreg} for $\bm{z}'$ for fixed $\bm{v}$ and express it in terms of $\bm{z}$ to get the result in Equation \ref{eq:zstardirectedreg}. Substituting this result back in \ref{eq:directedl1reformulationreg},

\begin{equation}
    \begin{split}
        \label{eq:directedz1l1reg}
{\phi_{l_1,l_1}^{D}}^2(\bm{\gamma}, \bm{\epsilon}) &= \max_{\bm{v} \in \mathcal{B}^{p}} \sum_{i= 1}^{p} [ |\tilde{\bm{c}}_i^{\top}\bm{v} + 2\tilde{\bm{x}}_i^{\top} \bm{y}| - \gamma_1 I_{( i \leq p_1  )} - \gamma_2 I_{(p_1 < i)} ]_+^2\\
&= \max_{\bm{v} \in \mathcal{S}^{p}} \sum_{i= 1}^{p} [ |\tilde{\bm{c}}_i^{\top}\bm{v} + 2\tilde{\bm{x}}_i^{\top} \bm{y}| - \gamma_1 I_{( i \leq p_1  )} - \gamma_2 I_{(p_1 < i)} ]_+^2
    \end{split}
\end{equation}

The last line follows from the fact that the objective function is convex, and the maximization is over a convex set, therefore the maxima are located on the boundary.

\end{proof}

Parallel to the Corollary \ref{cor:sparsitydirectedcor}, we can find the relationship between the sparsity pattern $\bm{\tau} \in \mathbb{R}^{p}$, and $\bm{v}^*$. 

\begin{corollary}
\label{cor:sparsitydirectedreg}
Solving \ref{eq:directedz1l1reg} for $\bm{v}^*$ given $\bm{\gamma}$ and $\bm{\epsilon}$, 

$$|\tilde{\bm{c}}_i^{\top}\bm{v}^* + 2\tilde{\bm{x}}_i^{\top} \bm{y}| \leq \gamma_1 I_{( i \leq p_1  )} + \gamma_2 I_{(p_1 < i)} \Rightarrow \tau_i = 0$$

\end{corollary}

\begin{proof}
According to Equation \ref{eq:zstardirectedreg} of Theorem \ref{thm:singlestagereg}, 

\begin{equation}
\label{eq:directedregsparsityl1}
\begin{split}
&|\bm{c}_i^T\bm{z}_1^* + \epsilon_2 \bm{x}_{2i}^{\top}\bm{y} | \leq \gamma_1 I_{( i \leq p_1  )} + \gamma_2 I_{(p_1 < i)}\\
&\Rightarrow [ | \tilde{\bm{c}}_i^T\bm{v}^* + 2 \tilde{\bm{x}}_{i}^{\top}\bm{y} | - \gamma_1 I_{( i \leq p_1  )} - \gamma_2 I_{(p_1 < i)} ]_+  = 0\\
&\Rightarrow \tau_{i}^* = 0
\end{split}
\end{equation}
We can go further and show that we can talk about $\bm{\tau}$ without solving for $\bm{v}^*$,

\begin{equation}
    | \tilde{\bm{c}}_i^T\bm{v}^* + 2 \tilde{\bm{x}}_{i}^{\top}\bm{y} | \leq \| \tilde{\bm{c}}_i\|_2 + 2\| \tilde{\bm{x}}_{i}\|_2
\end{equation}
Hence, 

\begin{equation}
    \tau_i = 0 \quad if \quad
   \| \tilde{\bm{c}}_i\|_2 + 2\| \tilde{\bm{x}}_{i}\|_2 \leq \gamma_1 I_{( i \leq p_1  )} + \gamma_2 I_{(p_1 < i)}, \quad for \quad i = 1, \ldots, p.
\end{equation}
\end{proof}

So far in Sections \ref{subsec:multimodal} and\ref{subsec:directed}, new approaches to Multi-View sCCA and Directed sCCA were introduced. The former was proposed to compute the canonical directions when we have more than two sets of variables, while the latter was proposed to direct the canonical directions towards an accessory direction.



\begin{proposition}
\label{cor:equivalent}
The Directed sCCA approach in \ref{subsubsec:singlestage}.a is equivalent to the approach in \ref{subsubsec:singlestage}.b assuming an orthogonal design matrix, i.e. $cov(\bm{X}_i) = \bm{I}_{p_i}$, and both are equivalent to the Multi-View sCCA approach where the inputs are three views $\bm{X}_1, \bm{X_2}$ and $\bm{y}$.
\end{proposition}

\begin{proof}
Assuming an orthogonal design, 

\begin{equation}
    \label{eq:equiv1}
    \min_{\bm{z} \in \mathcal{B}^p} \| \bm{y} - \bm{X}\bm{z} \|_2^2 = \min_{\bm{z} \in \mathcal{B}^p} \bm{y}^{\top}\bm{y} - 2\bm{y}^{\top}\bm{X}\bm{z} + \bm{z}^{\top}\bm{X}^{\top}\bm{X}\bm{z} = \max_{\bm{z} \in \mathcal{B}^p} \bm{y}^{\top}\bm{X}\bm{z} = \max_{\bm{z} \in \mathcal{B}^p} \langle \bm{y} , \bm{X}\bm{z} \rangle
\end{equation}

Hence programs \ref{eq:singlestagecor} and \ref{eq:singlestagereg} are equivalent. Now considering the multi-view approach for this problem,

\begin{equation}
\label{eq:equiv3}
\begin{split}
    \phi_{l_x}^M(\bm{\Gamma}) &= \max_{\substack{\bm{z}_i \in \mathcal{B}^{p_i}\\ \forall i = 1, \ldots, 3 }} \sum_{r<s = 2}^{3} \bm{z}_r^T\bm{C}_{rs}\bm{z}_s - \sum_{s = 1}^3  \sum_{\substack{r = 1 \\ r \neq s }}^{2} \Gamma_{sr} \| \bm{z}_s \|_1 \\
    &= \max_{\substack{\bm{z}_i \in \mathcal{B}^{p_i}\\ \forall i = 1, 2, 3 }} \bm{z}_1^T\bm{X}_{1}^{\top}\bm{X}_2\bm{z}_2 + \bm{z}_1^T\bm{X}_{1}^{\top}\bm{y}\bm{z}_3 +\bm{z}_2^T\bm{X}_{2}^{\top}\bm{y}\bm{z}_3 - \sum_{s = 1}^3  \sum_{\substack{r = 1 \\ r \neq s }}^{2} \Gamma_{sr} \| \bm{z}_s \|_1 \\
    &= \max_{\substack{\bm{z}_i \in \mathcal{B}^{p_i}\\ \forall i = 1, 2}} \bm{z}_1^T\bm{X}_{1}^{\top}\bm{X}_2\bm{z}_2 + \bm{z}_1^T\bm{X}_{1}^{\top}\bm{y} +\bm{z}_2^T\bm{X}_{2}^{\top}\bm{y} - \Gamma_{12}\| \bm{z}_1 \|_1 -  \Gamma_{21}\| \bm{z}_2 \|_1\\
\end{split}
\end{equation}
where the last line follows from the fact that $p_3 = 1$, so $\bm{z}_3^* = 1$. Equation \ref{eq:equiv3} is identical to \ref{eq:singlestagecor} for $\epsilon_1 = \epsilon_2 = 1$.

\end{proof}

\section{\texttt{MuLe}}
\label{sec:mule}

In this section we propose algorithms to solve the optimization programs introduced in Sections \ref{sec:convex} and \ref{sec:ext}. We also address the problem of initialization and hyper-parameter tuning. Our proposed algorithms are generally two-stage algorithms; in the first stage we find the sparsity patterns, $\bm{\tau}_i \in \{ 0,1 \}^{p_i}, \quad i = 1, \ldots, m$, of the optimal canonical directions via concave minimization programs introduced before, and in the second stage we shrink the covariance matrices using the sparsity patterns, $[\bm{C}_{ij}']_{rs} = [\bm{C}_{ij}]_{\tau_i^{(r)} \tau_j^{(s)}} $, where $\tau_i^{(r)}$ is the $r-th$ non-zero element of $\bm{\tau}_i$ or $r-th$ active element of $\bm{z}_i^*$, and solve the CCA problem using any \textit{Generalized Rayleigh Quotient} maximizer.

\begin{remark}
\label{rmk:seqtau}
In order to compute $\bm{\tau}_i$ for $i = 1, \ldots, m$, we start by computing $\bm{\tau}_m$, using which we shrink $\bm{C}_{im} \quad \forall i \neq m$ to $[\bm{C}_{im}']_{rs} = [\bm{C}_{im}]_{r\tau_m^{(s)}}$. This in turn shrinks the search space on $\bm{z}_m$ when computing $\bm{\tau}_i$, $i \neq m$. We perform the same shrinkage sequentially as we move down towards $\bm{\tau}_1$, shrinking the search space significantly each time. This sequential shrinkage, not only decreases computational cost drastically, it is also very useful in specifically very high-dimensional settings, since as with each shrinkage, we are directing successive solutions away from the normal cones of the preceding one. This might explain superior stability of our algorithm demonstrated in Section \ref{sec:sim}.
\end{remark}

Collecting from previous sections, the main differentiating characteristic of our approach is that we cast the problem of finding the sparsity patterns of the canonical directions as a maximization of a convex objective over a convex set, which is equivalent to the following \textit{Concave Minimization} problem,

\begin{equation}
    \phi^* = \max_{\bm{z} \in \mathbb{R}^p} f(\bm{z}) = \min_{\bm{z} \in \mathbb{R}^p}  {-f(\bm{z})}
\end{equation}

where $f:\mathbb{R}^p \rightarrow \mathbb{R}$ is a convex function. Consult \cite{mangasarian1996machine} and \cite{benson1995concave} for an in-depth treatment of this class of programs. \cite{journe:nesterov} propose a simple gradient ascent algorithm for this problem, for which they provide step-size convergence results. Considering these results as well as its empirical performance in terms of convergence and small memory foot-ptint, we also decided to use the following first-order method, 

\vspace{\baselineskip}

\begin{algorithm}[H]
\KwData{$\bm{z}_0 \in \mathcal{Q}$}
\KwResult{$\bm{z}^*  = \argmax_{\bm{z} \in \mathcal{Q}} f(\bm{z}) $}
$k \leftarrow 0$\\
\While{convergence criterion is not met}{
	$\bm{z}_{k+1} \leftarrow \argmax_{x \in \mathcal{Q}} (f(z_k) + (x - z_k)^Tf'(z_k)) $ \\
	$k \leftarrow k +1$
}
\caption{A first-order optimization method.}
\label{alg:1st}
\end{algorithm}

\vspace{\baselineskip}

What follows in this section, is the application of Algorithm \ref{alg:1st} to the programs proposed so far in this paper.


\subsection{\texorpdfstring{$l_1$}{TEXT}-Regularized Algorithm}
\label{subsec:l1regalg}

Applying algorithm \ref{alg:1st} to the problem in Program \ref{eq:shrink}.

\vspace{\baselineskip}

\begin{algorithm}[H] 
 \KwData{Sample Covariance Matrix $\bm{C}_{12}$\\  \quad \qquad    $l_1$-penalty parameter $\gamma_2$\\ \quad \qquad Initial value $\bm{z}_1 \in \mathcal{S}^{p_1}$}
 \KwResult{ $\bm{\tau}_2$, optimal sparsity pattern for $\bm{z}_2^*$}
 initialization\;
 \While{ convergence criterion is not met }{
	$\bm{z}_1 \leftarrow \sum_{i = 1}^{p_2} [| \bm{c}_i^{\top}\bm{z}_1 | - \gamma_2]_+ sgn(\bm{c}_i^{\top}\bm{z}_1)\bm{c}_i$\\
    $\bm{z}_1 \leftarrow \frac{\bm{z}_1 }{\| \bm{z}_1 \|_2}$}
 
 Output $\bm{\tau}_2 \in \{0,1\}^{p_2}$ where $\tau_{2i} = 0 $ if $|\bm{c}_i^{\top}\bm{z}_1^*| \leq \gamma_2$ and 1 otherwise.\\
 
 \caption{ \texttt{MuLe} algorithm for optimizing Program \ref{eq:shrink} }
 \label{alg:2nd}
\end{algorithm}

\vspace{\baselineskip}

Once the sparsity pattern $\bm{\tau}_2$ is found, we shrink the covariance matrix to $\bm{C}_{12}' \in \mathbb{R}^{p_1 \times | \bm{\tau}_2|}$, as prescribed at the beginning of this section, and apply Algorithm \ref{alg:1st} to ${\bm{C}_{12}'}^{\top}$ to find $\bm{\tau}_1$. Now we shrink the sample covariance matrix once more to $\bm{C}_{12}^{''} \in \mathbb{R}^{|\bm{\tau_1}| \times |\bm{\tau}_2|}$. For large enough sparsity parameters, this matrix is no more rank-deficient, and we can use conventional SVD or CCA methods to fill in the active elements of $\bm{z}_i$, i.e. solve for the leading singular vectors or canonical covariates of this much smaller matrix.

\subsection{\texorpdfstring{$l_0$}{TEXT}-Regularized Algorithm}
\label{subsec:l0regalg}

Now, we use Algorithm \ref{alg:1st} to optimize Program \ref{eq:shrinkl0}.

\vspace{\baselineskip}

\begin{algorithm}[H] 
 \KwData{Sample Covariance Matrix $\bm{C}_{12}$\\  \quad \qquad    $l_1$-penalty parameter $\gamma_2$\\ \quad \qquad Initial value $\bm{z}_1 \in \mathcal{S}^{p_1}$}
 \KwResult{ $\bm{\tau}_2$, optimal sparsity pattern for $\bm{z}_2^*$}
 initialization\;
 \While{ convergence criterion is not met }{
	$\bm{z}_1 \leftarrow \sum_{i= 1}^{p_2} [ ( \bm{c}_i^{\top} \bm{z}_1 )^2 - \gamma_2  ]_+ \bm{c}_i^{\top}\bm{z}_1\bm{c}_i$\\
    $\bm{z}_1 \leftarrow \frac{\bm{z}_1 }{\| \bm{z}_1 \|_2}$}
 
 Output $\bm{\tau}_2 \in \{0,1\}^{p_2}$ where $\tau_{2i} = 0 $ if $(\bm{c}_i^{\top}\bm{z}_1^*)^2 \leq \gamma_2$ and 1 otherwise.\\
 
 \caption{ \texttt{MuLe} algorithm for optimizing Program \ref{eq:shrinkl0} }
 \label{alg:2ndl0}
\end{algorithm}

\vspace{\baselineskip}

Similar to \ref{subsec:l1regalg}, we perform successive shrinkage and find $\bm{\tau}_1$ in the nest step by applying Algorithm \ref{alg:2ndl0} on the shrunk matrix ${\bm{C}_{12}'}^{\top}$.

\subsection{Algorithm Complexity}
Perhaps the most appealing characteristic of our proposed algorithm is its significantly lower time complexity compared to other state of the art algorithms. Here we analyze the time complexity of \texttt{MuLe} and compare it to the most common algorithm for \textit{sCCA} which is the alternating first order optimization, e.g. \cite{waaijenborg}, \cite{parkhomenkoSCCA}, \cite{witten:tibshirani:2009}, for which we use the umbrella term \texttt{sSVD} here. Following the set-up thus far, assume we have observed $\bm{X}_1 \in \mathbb{R}^{n \times p_1}$ and $\bm{X}_2 \in \mathbb{R}^{n \times p_2}$ and we wish to recover sparse canonical loading vectors $\bm{z}_1 \in \mathbb{R}^{p_1}$ and $\bm{z}_2 \in \mathbb{R}^{p_2}$. 
In order to create more intuition about the speed-up consider a hypothetical algorithm which uses power method to solve a SVD problem and finally simply uses hard-thresholding to create sparse loading vectors. We will call this algorithm \texttt{pSVDht}. Also consider another hypothetical algorithm called \texttt{sSVDht} which performs the alternating maximization and similarly induces sparsity by hard-thresholding.

\begin{proposition}\label{proposition:complexity}
Time complexity of each iteration of \texttt{MuLe} is smaller than that of \textit{pSVDht} if $n < min\{p_1, p_2\}$ and $p_1 \sim p_2$.
\end{proposition}

\textbf{Proof.} The proof of Proposition \ref{proposition:complexity} is presented in Appendix \ref{app:proposition:complexity}.

\begin{proposition}
\label{proposition:pmd}
The time complexity of each $(z_1,z_2)$ update of the \texttt{MuLe} algorithm, i.e. Algorithm \ref{alg:2nd}, is significantly lower  than that of the \texttt{sSVD} algorithm, \cite{witten:tibshirani:2009} Algorithm 3. 
\end{proposition}

\textbf{Proof.} A simple proof is provided in Appendix \ref{app:proposition:pmd}.

\subsection{Sparse Multi-View CCA Algorithm}
\label{subsec:multimodalalg}

Our sparse multi-view formulation offered in Program \ref{eq:multiznoptsimple} scales linearly with the number of views, which along with the immense shrinkage of the search domain as a result of our concave minimization program results in considerable reduction in convergence time. Below is our proposed gradient ascent algorithm for finding $\bm{\tau}_i \in \{1,2 \}^{p_i}$, $i = 1, \ldots, m$.

\vspace{\baselineskip}
\begin{algorithm}[H] 
 \KwData{Sample Covariance Matrices $\bm{C}_{rs}, \quad 1 \leq r < s \leq m$\\  \quad \qquad    Sparsity parameter matrix $\bm{\Gamma} \in [0,1]^{m \times m}$\\ \quad \qquad Initial values $\bm{z}_r \in \mathcal{S}^{p_r}, \quad 1 \leq r \leq m$}
 \KwResult{ $\bm{\tau}_s$, optimal sparsity pattern for $\bm{z}_s$}
 initialization\;
 \While{ convergence criterion is not met }{
 \For{$r = 1, \ldots, m$, $r \neq s$}{
	$\bm{z}_r \leftarrow \sum_{i= 1}^{p_s} [ | \sum_{ \substack{r = 1\\ r \neq s}}^m  \tilde{\bm{c}}_{rsi}^{\top} \bm{z}_r | - \sum_{\substack{r = 1\\ r \neq s}}^{m} \Gamma_{sr} ]_+ sgn( \sum_{ \substack{r = 1\\ r \neq s}}^m  \tilde{\bm{c}}_{rsi}^{\top} \bm{z}_r) \tilde{\bm{c}}_{rsi} +  \sum_{\substack{l = 1\\ l \neq r,s}}^{m} \tilde{\bm{C}}_{rl}\bm{z}_l $\\
    $\bm{z}_r \leftarrow \frac{\bm{z}_r }{\| \bm{z}_r \|_2}$
    }
 }
 
 Output $\bm{\tau}_s \in \{0,1\}^{p_s}$,  where $\tau_{si} = 0 $ if $| \sum_{ \substack{r = 1\\ r \neq s}}^m  \tilde{\bm{c}}_{rsi}^{\top} \bm{z}_r | \leq \sum_{\substack{r = 1\\ r \neq s}}^{m} \Gamma_{sr}$ and 1 otherwise.\\
 
 \caption{\texttt{MuLe} algorithm for optimizing Program \ref{eq:multiznoptsimple}}
 \label{alg:3rd}
\end{algorithm}
\vspace{\baselineskip}

Once $\bm{\tau}_s$ is computed we can use successive shrinkage to shrink $\tilde{\bm{C}}_{rs}$, $r = 1, \ldots, m$, $r \neq s$, per instructions provided in Remark \ref{rmk:seqtau}, to $\tilde{\bm{C}}_{rs}' \in \mathbb{R}^{p_r \times |\bm{\tau}_s|}$. We compute the rest of the sparsity patterns by repeating Algorithm \ref{alg:3rd} together with successive shrinkage.

Finally we shrink all covariance matrices to $\bm{C}_{rs}'' \in \mathbb{R}^{|\bm{\tau}_r| \times |\bm{\tau}_s|}$ using computed sparsity patterns. The second stage of our algorithm, as before, involves estimating the active elements of $\bm{z}_i^*$; for which we propose two algorithms, the \texttt{mCCA} algorithm, see Appendix \ref{app:mcca}, and the \texttt{mSVD} algorithm, see Appendix \ref{app:msvd}.

\subsection{Single Stage Sparse Directed CCA Algorithm}
\label{subsec:singlestageDirectedMule}
We proposed three approaches in \ref{subsec:directed} for \textit{Directed sCCA} problem; one two-stage, where we first perform variable selection and then perform sCCA on the covariance matrix of the selected variables, and two single-stage methods, where we direct the canonical covariates to align with certain outcome of subspace. For our proposed two-stage algorithm refer to the Appendix \ref{app:semimule}. Here we elaborate on our single-stage algorithms, starting with \ref{subsubsec:singlestage}.a, we apply our gradient ascent algorithm to Program \ref{eq:directedz1l1}. Once again we optimize it with no regards to the regularization term in the first stage.

\vspace{\baselineskip}

\begin{algorithm}[H] 
 \KwData{Sample Covariance Matrix $\bm{C}_{12}$\\  \quad \qquad    $l_1$ regularization parameter $\gamma_2$\\ \quad \qquad
 Alignment hyperparameters $(\epsilon_1, \epsilon_2) $\\ \quad \qquad
 Initial value $\bm{z}_1 \in \mathcal{S}^{p_1}$}
 \KwResult{ $\bm{\tau}_2$, optimal sparsity pattern for $\bm{z}_2^*$}
 initialization\;
 \While{ convergence criterion is not met }{
	$\bm{z}_1 \leftarrow \sum_{i = 1}^{p_2} [| \bm{c}_i^{\top}\bm{z}_1 + \epsilon_2\bm{x}_{2i}^{\top}\bm{y}| - \gamma_2]_+ sgn(\bm{c}_i^{\top}\bm{z}_1 + \epsilon_2\bm{x}_{2i}^{\top}\bm{y})\bm{c}_i + \epsilon_1 \bm{X}_1^{\top}\bm{y}$\\
    $\bm{z}_1 \leftarrow \frac{\bm{z}_1 }{\| \bm{z}_1 \|_2}$}
 
 Output $\bm{\tau}_2 \in \{0,1\}^{p_2}$ where $\tau_{2i} = 0 $ if $|\bm{c}_i^T\bm{z}_1^* + \epsilon_2 \bm{x}_{2i}^{\top}\bm{y} | \leq \gamma_2$ and 1 otherwise.\\
 
 \caption{ \texttt{MuLe} algorithm for optimizing Program \ref{eq:directedz1l1} }
 \label{alg:4th}
\end{algorithm}

\vspace{\baselineskip}

As before, to compute $\bm{\tau}_1$, we use successive shrinkage, and in the second stage we use conventional SVD or CCA to estimate the active entries. Regarding \ref{subsubsec:singlestage}.b, rather than an algorithm solving Program \ref{eq:directedz1l1reg}, we propose a simpler Algorithm which is identical to Algorithm \ref{alg:4th}, except that we $\bm{X}_i^{\top}\bm{y}$ with $\bm{\beta}_i$ for $i = 1,2$, similarly $\bm{x}_{ij}^{\top}\bm{y}$ with $\beta_{ij}$, which is the vector of coefficient estimates from regressing $\bm{y}$ on $\bm{X}_i$.

\subsection{Initialization \& Hyperparameter Tuning}
\label{subsec:initpartune}

\subsubsection{Initialization}
Concerning the initialization, we follow the suggestion of \cite{journe:nesterov} and choose an initial value $\bm{z}_{1,init}$ for which our algorithm is guaranteed to yield a sparsity pattern with at least one non-zero element. This initial value is chosen parallel to the column with the largest $L_2$ norm.

\begin{equation}
{\bm{z}_{1}}_{init} = \frac{\bm{c}_{i^*}}{\| \bm{c}_{i^*}\|_2}, \quad i^* = \argmax_{i\in \{1, \ldots, p_1\}} \| \bm{c}_{i}\|_2
\end{equation}

Where $\bm{c}_{i}$ is the $i$-th column of $\bm{C}_{12}$. Similarly, ${\bm{z}_{2}}_{init} = \bm{c}_{i^*}'/\| \bm{c}_{i^*}'\|_2$, where $\bm{c}_{i^*}'$ is the column of the transpose of the shrunk covariance matrix.

\subsubsection{Hyperparameter Tuning}

Algorithms \ref{alg:2nd}-\ref{alg:4th} involve choosing hyperparameters $\bm{\gamma}$ and $\bm{\epsilon}$. Here we propose two algorithm for choosing the optimal sparsity parameters, $\gamma_i$; they are easily extendable to tuning alignment parameters $\epsilon_i$. But we first need to choose a performance criteria in order to compare different choices of parameters. \cite{witten:tibshirani:2009} choose penalty parameters which best estimate entries that were randomly removed from the covariance matrix, while some choose them by comparing the Frobenius norms of the reconstructed covariance matrices subtracted from the original matrix. These choices are effectively imposed due to solving a penalized SVD instead of the sCCA problem. However, since we solve the CCA problem in the second stage of our algorithm, we use the canonical correlation, $\rho_{\gamma_1, \gamma_2}(\bm{X}_1^{\top}\bm{z}_1, \bm{X}_2^{\top}\bm{z}_2)$, as our measure, which serves our objective more properly.

Algorithm \ref{alg:hyperparametercv} performs hyperparameter tuning using the $k$-fold cross-validation method, which is widely common in sCCA literature. 

\vspace{\baselineskip}
\begin{algorithm}[H]
 \KwData{Sample matrices $\bm{X}_i \in \mathbb{R}^{n \times p_i}$, $i = 1,2$ \\  \quad \qquad    Sparsity parameters $\gamma_i$, $i = 1,2$\\ \quad \qquad Initial values $\bm{z}_i \in \mathcal{S}^{p_i}$, $i = 1,2$ \\  \quad \qquad Number of folds $K$}
 \KwResult{ $\rho_{CV}(\gamma_1,\gamma_2)$ the average cross-validated canonical correlation}

Let $\bm{X}_{ik}, \bm{X}_{i/k}, i = 1,2, j = 1, \ldots, K$ be the validation and training sets corresponding to the $k$-th fold, respectively.\\
 \For{k = 1, \ldots, K}{
 Compute $({\bm{z}_1^*}^{(k)}, {\bm{z}_2^*}^{(k)})$ on $\bm{X}_{1/k}, \bm{X}_{2/k}$ via proposed methods in \ref{subsec:l1regalg} or \ref{subsec:l0regalg} with sparsity hyperparameters $(\gamma_{1}, \gamma_{2})$\\
 $\rho^{(k)}(\gamma_1, \gamma_2) = corr(\bm{X}_{1k}{\bm{z}_1^*}^{(k)}, \bm{X}_{2k}{\bm{z}_2^*}^{(k)})$
 }
 $\rho_{CV}(\gamma_1, \gamma_2) = 1/K\sum_{k = 1}^K \rho^{(k)}(\gamma_1, \gamma_2)$ \\
 
\caption{Hyperparameter Tuning via $k$-Fold Cross-Validation}
\label{alg:hyperparametercv}
\end{algorithm}
\vspace{\baselineskip}

This approach has a significant shortcoming, specially in high-dimensional settings, though. The issue is that once the sparsity parameter is small enough, the fitted models return high correlation values, close to one, which makes the choice of best parameters inaccurate. To cope with this problem, we propose a second algorithm which performs a permutation test, where the null hypothesis is that the views $\bm{X}_i$ are independent. In order to reject the null, the canonical correlation computed from the matched samples must be significantly higher than the average canonical correlation computed from the permuted samples. To this end, we propose Algorithm \ref{alg:hyperparameterpermutation}. Given a grid of hyperparameters, the tuple which minimizes the $p$-value is chosen.



Algorithm \ref{alg:hyperparametercv} performs hyperparameter tuning using the $k$-fold cross-validation method, which is widely common in sCCA literature. 

\vspace{\baselineskip}
\begin{algorithm}[H]
 \KwData{Sample matrices $\bm{X}_i \in \mathbb{R}^{n \times p_i}$, $i = 1,2$ \\  \quad \qquad     Sparsity parameters $\gamma_i$, $i = 1,2$\\ \quad \qquad Initial values $\bm{z}_i \in \mathcal{S}^{p_i}$, $i = 1,2$ \\  
 \quad \qquad Number of permutations $P$}
 \KwResult{ $p_{\gamma_1, \gamma_2}$ the evidence against the null hypothesis that the canonical correlation is not lower when $X_i$ are independent.}

Compute $({\bm{z}_1^*}, {\bm{z}_2^*})$ on $\bm{X}_{1}, \bm{X}_{2}$ via proposed methods in \ref{subsec:l1regalg} or \ref{subsec:l0regalg} with sparsity hyperparameters $(\gamma_{1}, \gamma_{2})$\\
$\rho(\gamma_1, \gamma_2) = corr(\bm{X}_{1}{\bm{z}_1^*}, \bm{X}_{2}{\bm{z}_2^*})$\\

 \For{p = 1, \ldots, P}{
 Let $\bm{X}_{1}^{(p)}$ be a row-wise permutation of $\bm{X}_1$\\
 Compute $({\bm{z}_1^*}^{(p)}, {\bm{z}_2^*}^{(p)})$ on $\bm{X}_{1}^{(p)}, \bm{X}_{2}$ via proposed methods in \ref{subsec:l1regalg} or \ref{subsec:l0regalg} with sparsity hyperparameters $(\gamma_{1}, \gamma_{2})$\\
 $\rho_{perm}^{(p)}(\gamma_1, \gamma_2) = corr(\bm{X}_{1}^{(p)}{\bm{z}_1^*}^{(p)}, \bm{X}_{2}{\bm{z}_2^*}^{(p)})$
 }
 $p_{\gamma_1, \gamma_2} = 1/P \sum_{p = 1}^P I(\rho_{perm}^{(p)} > \rho)$ \\
 
\caption{Hyperparameter Tuning via Permutation Test}
\label{alg:hyperparameterpermutation}
\end{algorithm}
\vspace{\baselineskip}

\section{Experiments}
\label{sec:sim}

In this section we compare and evaluate our proposed algorithm \texttt{MuLe} along with few other sparse CCA algorithms. To perform an inclusive comparison, we tried to choose representatives from different approaches. As argued in \ref{subsub:lasso}, optimization problems introduced in \cite{witten:tibshirani:2009}, \cite{parkhomenkoSCCA}, \cite{waaijenborg} are equivalent. The methods used here for comparison are the \textit{Penalized Matrix Decomposition} proposed in \cite{witten:tibshirani:2009} which is implemented in the \texttt{PMA} package, and also a ridge regularized CCA, noted here as \texttt{RCCA}. In order to benchmark \texttt{MuLe} comprehensively, simple \texttt{SVD} and \texttt{SVDthr}, which is simply soft-thresholded SVD, are also included. Note that as mentioned before almost all sparse CCA algorithms try to solve a penalized singular value decomposition problem, whereas we solve a CCA problem in the second stage. In \ref{subsec:rankOne} and \ref{subsec:sim:randomdata}
we first establish the accuracy of our algorithm, then we compare compute and compare few characteristic curves regarding stability of our algorithm. We also compare out Multi-View Sparse CCA algorithm with other popular algorithm, the results of which is included in Appendix \ref{app:subsec:multiview}.

\subsection{A Rank-One Sparse CCA Model} 
\label{subsec:rankOne}
Consider a CCA problem where $\bm{X}_1$ and $\bm{X}_2$ are generated using the following rank-one model,

\begin{equation}
\bm{X}_1  = (\bm{z}_1 + \bm{\epsilon}_1)\bm{u}^{\top}, \quad \bm{X}_2  = (\bm{z}_2 + \bm{\epsilon}_2)\bm{u}^{\top}
\end{equation}

where $\bm{z}_1 \in \mathbb{R}^{500}$ and $\bm{z}_2 \in \mathbb{R}^{400}$ have the following sparsity patterns,

\begin{equation}
\begin{split}
\bm{z}_{1} &= \bigg[\underbrace{1, \ldots ,1}_{25}\quad \underbrace{-1, \ldots ,-1}_{25}\quad \underbrace{0, \ldots, 0}_{450}\bigg]\\
\bm{z}_{2} &= \bigg[\underbrace{1, \ldots ,1}_{25}\quad \underbrace{-1, \ldots ,-1}_{25} \quad \underbrace{0, \ldots, 0}_{350}\bigg]
\end{split}
\end{equation}

$\bm{\epsilon}_1 \in \mathbb{R}^{400}$ and $\bm{\epsilon}_2 \in \mathbb{R}^{500}$ are added Gaussian noise.

\begin{equation}
\begin{split}
\bm{\epsilon}_1 &\sim \mathcal{N}(0, \sigma^2), \forall i = 1, \ldots 500,\\
\bm{\epsilon}_2 &\sim \mathcal{N}(0, \sigma^2), \forall i = 1, \ldots 400,
\end{split}
\end{equation}

and

\begin{equation}
\bm{u}_i \sim \mathcal{N}(0,1), \forall i = 1, \ldots, 50.
\end{equation}

Figure \ref{fig:rankOne} compares \texttt{MuLe}'s performance to the methods mentioned above. The noise amplitude, $\sigma$ was set to $0.2$, in order to more significantly differentiate between the methods. It is evident that \texttt{MuLe} successfully identified the underlying sparse model since both the sparsity pattern and the value of the coefficients were estimated quite accurately, while \texttt{PMA} failed to estimate the coefficient sizes accurately. Note here that, our simple cross-validation parameter tuning resulted in accurate identification of the canonical directions while using the same procedure on \texttt{PMA} resulted in cardinalities far from the specified model. Hence, the sparsity parameters for the latter method were chosen by trial-and-error to match model's sparsity pattern.

\begin{figure}
\includegraphics[width=1\textwidth]{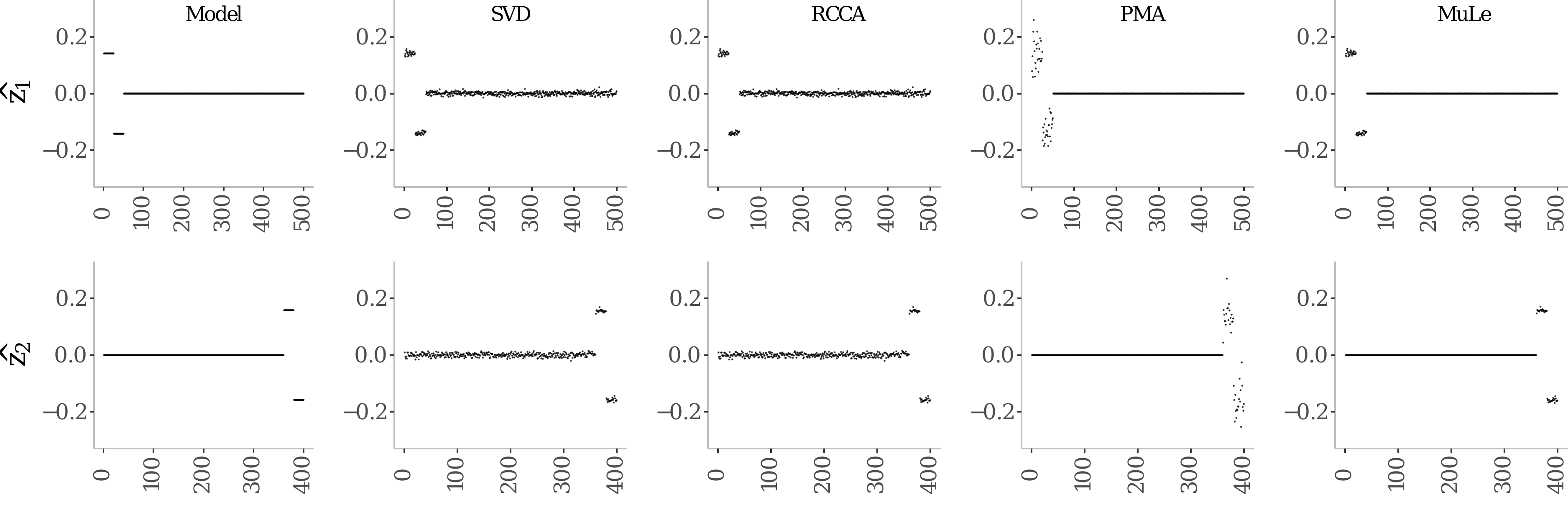}
\caption{Comparing performance of different sCCA approaches in recovering the sparsity pattern and estimating active elements of the canonical directions. The \textit{Model} or ``true" canonical directions are plotted in the leftmost plot.}
\label{fig:rankOne}
\end{figure}

Under the same setting, but varying level of noise $\sigma$, we compute the cosine of the angle between the estimated, $\hat{\bm{z}}_i$, and true, $\bm{z}_i$, canonical directions, $cos(\theta_i) = |\langle \bm{z}_i, \hat{\bm{z}}_i \rangle|$ for $i = 1, 2$ via the methods utilized in Figure \ref{fig:rankOne}. We plotted the results in Figure \ref{fig:snr} for both canonical directions; according to which, \texttt{MuLe} outperforms other methods, especially the alternating method of \cite{witten:tibshirani:2009}, throughout the range of noise amplitude. \texttt{PMA} uniquely shows a lot of volatility in its solution. The built-in parameter tuning also misspecified the correct sparsity parameters, but providing correct hyperparameters manually also did not help much. Actually, our test shows that a simple thresholding algorithm like \texttt{SVDthr} outperforms \texttt{PMA} both in terms of support recovery and direction estimation.


\begin{figure}
\includegraphics[width = 1\textwidth]{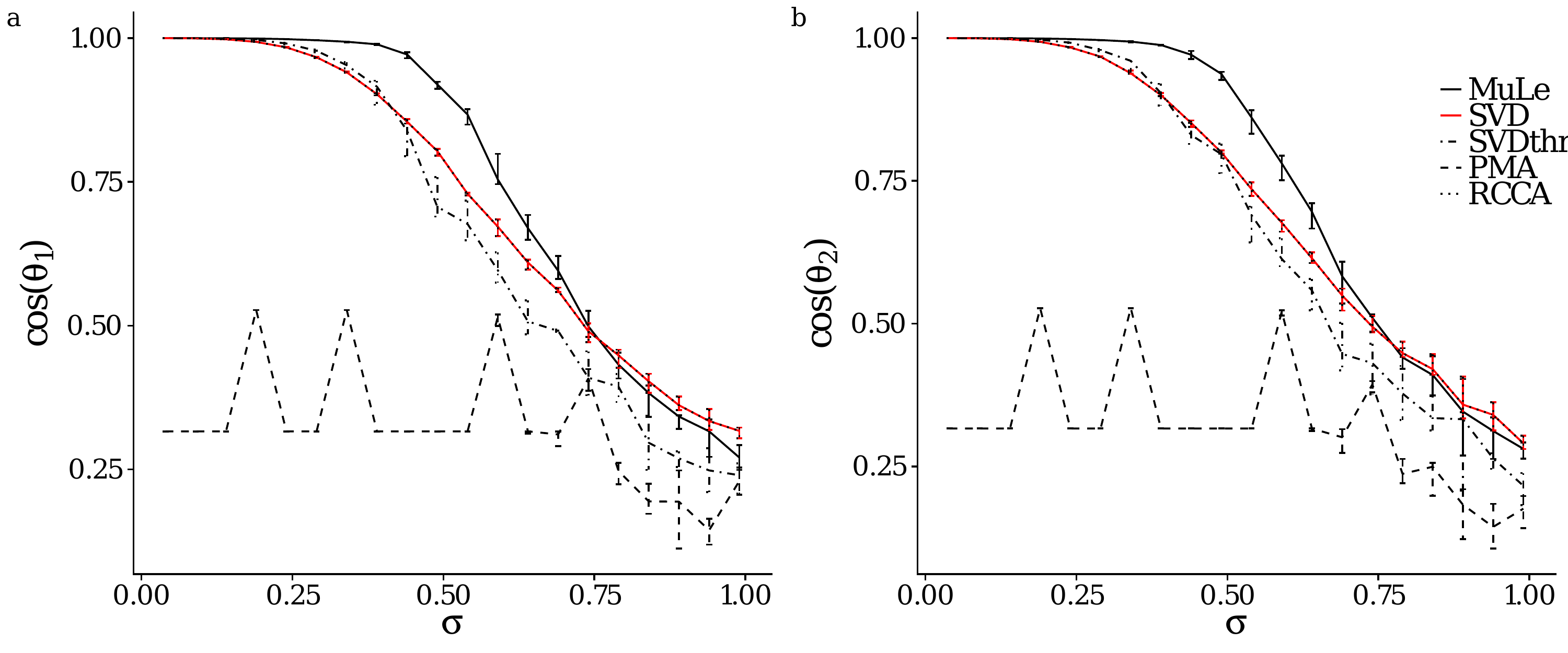}
\caption{The cosine of the angle between the estimated and true canonical directions, $cos(\theta_i) = |\langle \hat{\bm{z}}_i, \bm{z}_i \rangle|$ computed for both datasets.} 
\label{fig:snr}
\end{figure}

But perhaps the most important piece of information one looks for in high-dimensional multi-view studies is the interpretability of the estimated canonical directions. Therefore, ultimately the decisive criteria in choosing the best approach is determined by how well they uncover the ``true" underlying sparsity pattern or simply put, how accurately a model performs variable selection. To this end, variable selection accuracy of each method is plotted against the noise amplitude in Fig. \ref{fig:card} as the fraction of the support of $\bm{z}_i$, $i \in \{1,2\}$ discovered, here denoted as $\eta_i$, vs. the noise amplitude, $\sigma$. As before \texttt{MuLe} performs significantly better than other methods throughout the noise amplitude range. 

\subsection{Solution Stability on Data Without Underlying Sparse CCA Model}
\label{subsec:sim:randomdata}

In the following simulations, $\bm{X}_1$ and $\bm{X}_2$ are generated by sampling from $\mathcal{N}(\bm{0}_{p_i}, \bm{I}_{p_i}), i \in \{1,2\}$. The main purpose of this section is to demonstrate the stability of the solution paths while comparing the quality of the solutions of different algorithms as a function of the cardinality of the canonical loadings. The motivation behind this simulation is that the solution of a stable algorithm must grow more similar to the non-sparse CCA solution. Therefore, while setting the sparsity parameter equal to zero for one canonical direction, for an array of sparsity parameters we compute the correlation of the estimated direction with the corresponding direction from the CCA solution, as well as the estimated canonical correlation for the same setting.

The results of the aforementioned simulation is presented in Figure \ref{fig:card}. According to our results \texttt{MuLe} is consistently more correlated with the CCA solution and for $(\gamma_1, \gamma_2) = (0,0)$, it solves the CCA problem whereas \texttt{PMA} by far does not show the same solution stability. Were columns of $\bm{X}_i$ more correlated, \texttt{PMA} and \texttt{SVDThr} would have resulted in even worse solutions.

\begin{figure}
\includegraphics[width = 1\textwidth]{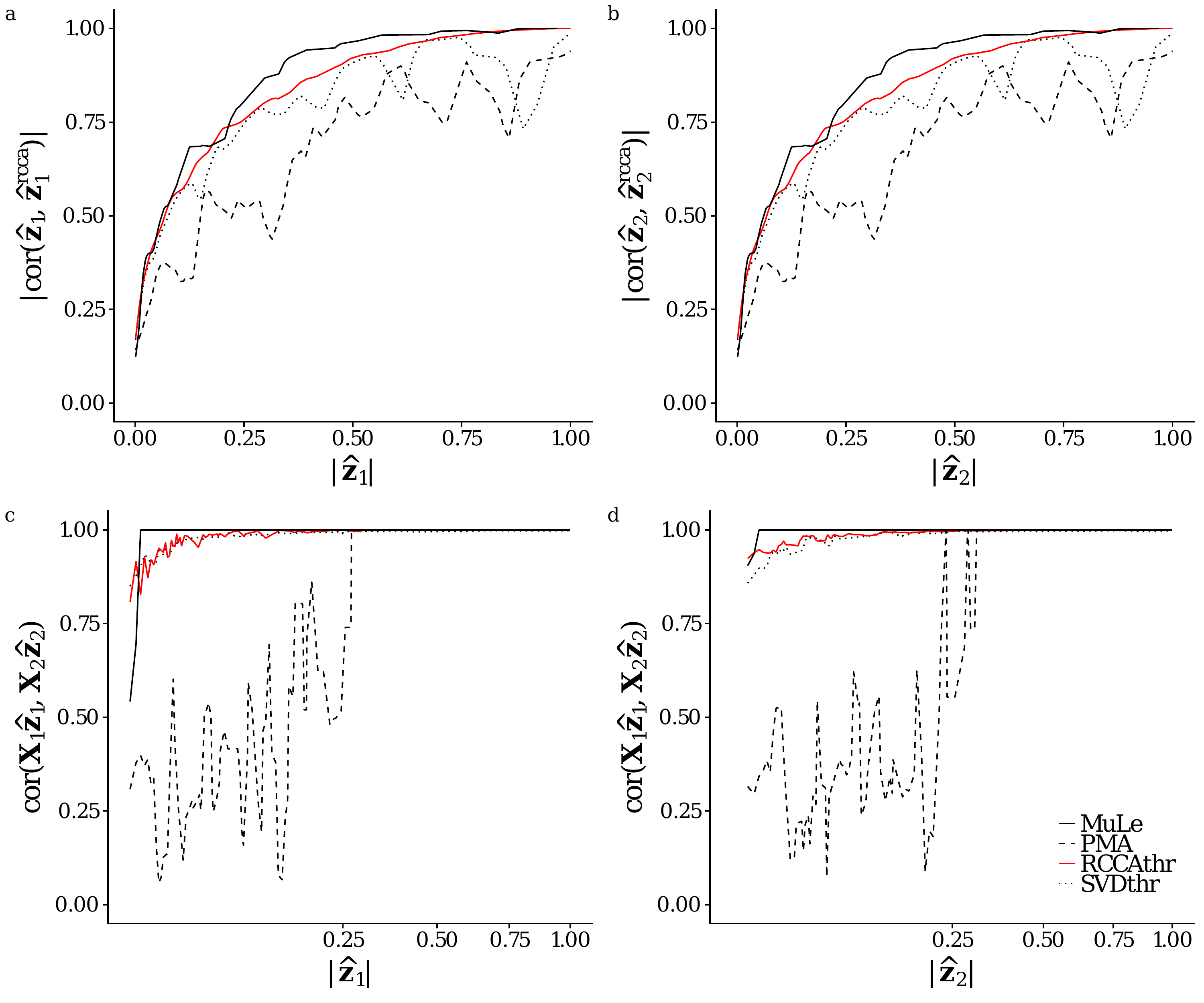}
\caption{The correlation between the estimated sparse canonical direction and the direction obtained from CCA. (a,b) and the estimated canonical correlation as a function of the cardinality of the estimated direction. (c,d)}
\label{fig:card}
\end{figure}





In the next section we utilize \texttt{MuLe} to discover correlation structures in a genomic setting.


\section{Fruitfly Pesticide Exposure Multi-Omics}
\label{sec:reda}

One of the drivers for the development of our method was the rise of multi-omics analysis in functional genomics, pharmacology, toxicology, and a host of related disciplines. Briefly, multiple ``omic" modalities, such as transcriptomics, metabolomics, metagenomics, and many other possibilities, are executed on matched (or otherwise related) samples. An increasingly common use in toxicology is the use of transcriptomics and metabolomics to identify, in a single experiment, the genetic and metabolic networks that drive resilience or susceptibility to exposure to a compound[\cite{campos2018omics}]. We analyzed recently generated transcriptomics, metabolomics, and 16S DNA metabarcoding data generated on isogenic Drosophila (described in Brown et al. 2019, in preparation). In this experiment, fruit flies are separated into treatment and control groups, where treated animals are exposed to the herbicide Atrazine, one of the most common pollutants in US drinking water. Dosage was calculated as 10 times the maximum allowable concentration in US drinking water -- a level frequently achieved in surface waters (streams and rivers) and rural wells. 

Data was collected after 72 hours, and little to no lethality was observed. Specifically, male and female exposed flies were collected, whereafter mRNA, small molecular metabolites, and 16S rDNA (via fecal collection and PCR amplification of the V3/V4 region) was collected. RNA-seq and 16S libraries were sequenced on an \textit{Illumina MiSeq}, and polar and non-polar metabolites were assayed by direct injection tandem mass spectrometry on a \textit{Thermo Fisher Orbitrap Q Exactive}. Here, we compare 16S, rDNA and metabolites using \texttt{MuLe}, to identify small molecules associated with microbial communities in the fly gut microbiome. 

This is an intriguing question, as understanding how herbicide exposure remodels the gut microbiome, and, in turn, how this remodeling alters the metabolic landscape to which the host is ultimately exposed is a foundational challenge in toxicology. All dietary co-lateral exposures are "filtered through the lens" of the gut microbiome -- compounds that are rapidly metabolized by either the host system or the gut are experienced, effectively, at lower concentrations; the microbiome plays an important role in toxicodynamics.

We utilized the multi-view sparse CCA module of \texttt{MuLe} to find three-way associations in our study. Hyper-parameter was performed using our permutation test of Algorithm \ref{alg:hyperparameterpermutation} modified to lean towards more sparse models. Our analysis, see Figures \ref{fig:rnaFecalColonBiplot} and \ref{fig:fecalColonHclust}, revealed three principle axes of variation. The first groups host genes for primary and secondary metabolism, cell proliferation, and reproduction along with host metabolites related to antioxidant response. Intriguingly, all metabolites in this axis of variation derive from the linoleic acid pathway, part of the anti-oxidant defense system, which is known to be engaged in response to Atrazine exposure [\cite{sengupta2015hr96}]. Similarly, Glutathione S transferase D1 (GstD1), a host gene that varies along this axis, is a secondary metabolic enzyme that leverages glutathione to neutralize reactive oxygen species (eletrophilic substrates). Linoleic acid metabolites are known to strongly induce glutathione synthesis [\cite{arab2006conjugated}]. The primary metabolism gene, Cyp6w1 is strongly up-regulated in response to atrazine [\cite{sieber2009dhr96}], and here we see it is also tightly correlated with the anti-oxidant defense system. We see broad inclusion of cell proliferation genes (CG6770, CG16817, betaTub56D,) and genes involved in reproduction (the Chorion proteins, major structural components of the eggshell chorion, Cp15, Cp16, Cp18, Cp19, Cp38, and Vitelline membrane 26Aa (Vm26Aa)), and it is well known that flies undergo systematic repression of the reproductive system during exposure to environmental stress [\cite{brown2014diversity}]. Whether this reproductive signal is directly associated with linoleic acid metabolism and glutathione production is an intriguing question for future study. 

The second principle axis of variation groups a dominant microbial clade (Lactobacillales) along with a collection of host metabolites, and one gene of unknown function. The host metabolites fall principally on the phosphorylcholine metabolic pathway, which is known to be induced in a sex-specific fashion in response to atrazine in mammals, but, as far as we know, not previously reported in arthopods [\cite{holaskova2019long}] – which may be useful, as it expands the domain of mammalian adverse outcome pathways that can be modeled in Drosophila. 

The third and final principle axis includes two host genes – a cytochrome P450 (Cyp4g1) known to be involved in atrazine detoxification [\cite{sieber2009dhr96}], and a peptidase of unknown function (CG12374) – a minority microbial clade (Rhodospirillales, [\cite{chandler2011bacterial}]), and another collection of linoleic acid pathway metabolites, along with 1-Oleoylglycerophosphoinositol, a host metabolite derived from oleic acid. While the ostensible lack of known microbial metabolites is somewhat disappointing, it may also be that these were simply not assigned chemical IDs during the metabolite identification – a common challenge with untargeted chemistry. 

\begin{figure}
\centering
\includegraphics[width=1\textwidth, center]{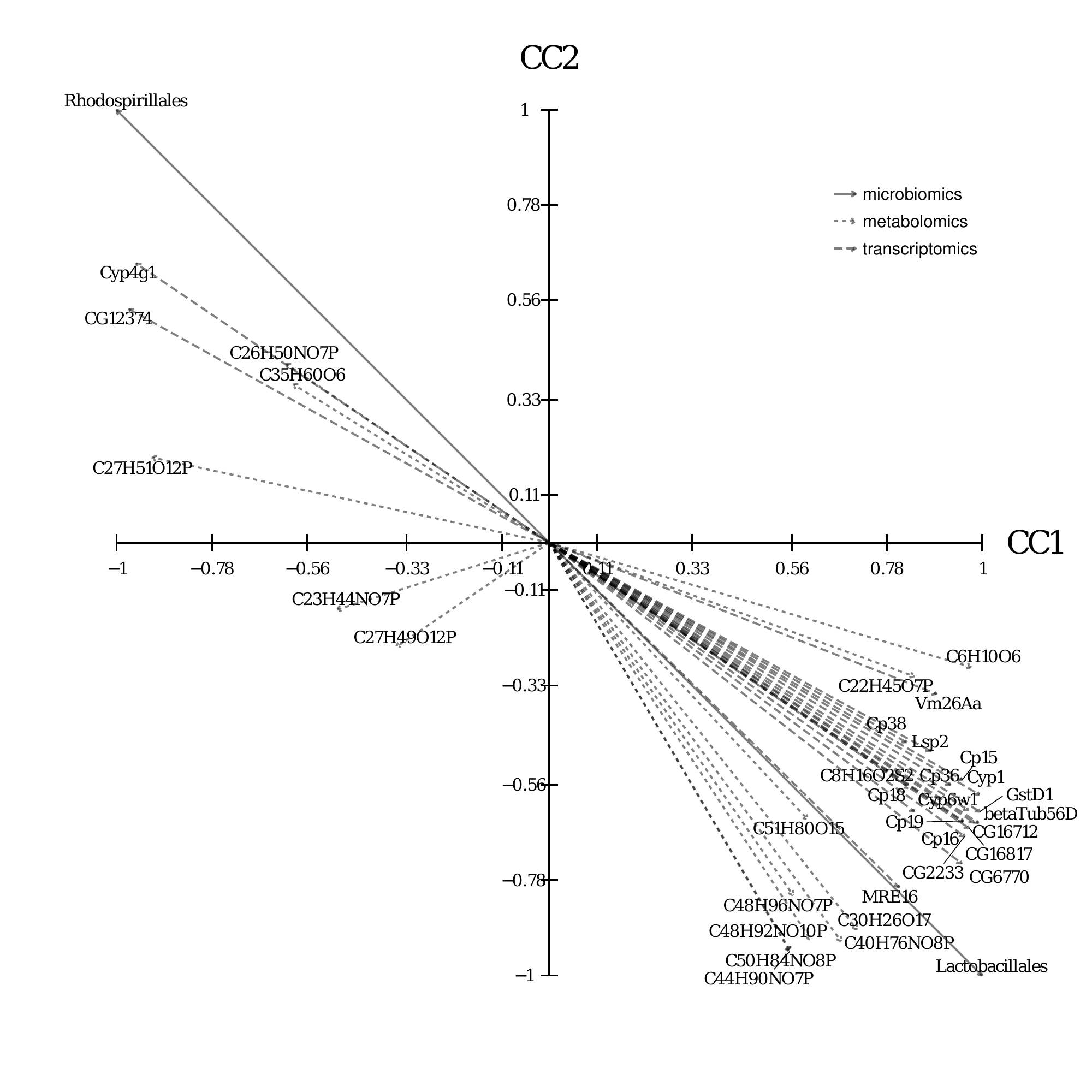}
\caption{CCA biplot of transcriptomic, microbiomic, and metabolomic datasets in Drosophila Atrazine exposure experiment.}
\label{fig:rnaFecalColonBiplot}
\end{figure}

\begin{figure}
\centering
\includegraphics[width=1\textwidth]{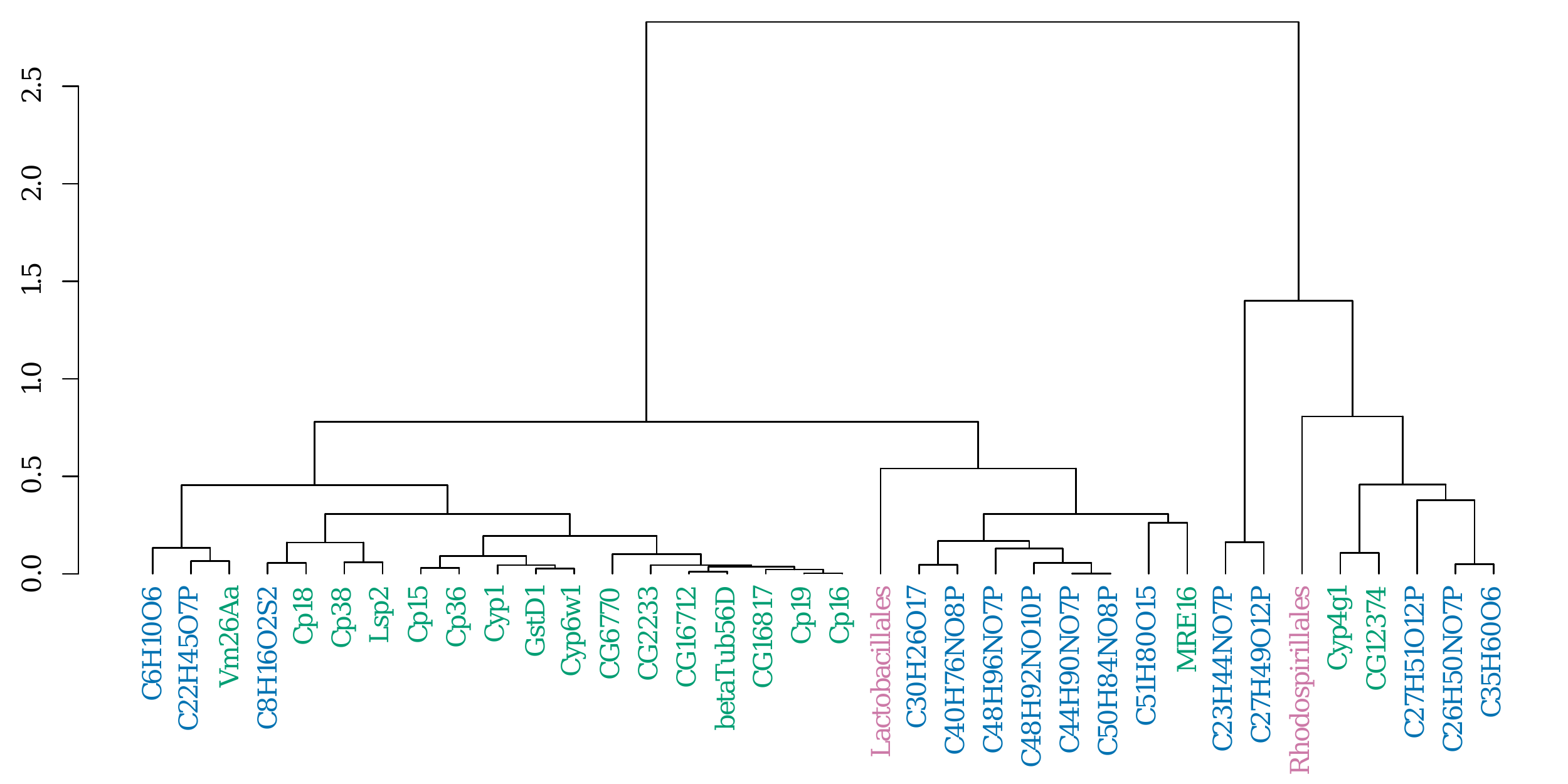}
\caption{Hierarchical clustering of the first two pairs of canonical directions.}
\label{fig:fecalColonHclust}
\end{figure}

In order to verify that the primary effect captured in our canonical directions are co-variations associated with the treatment effect, and not that of sex, exposure length etc., we also projected our samples on to the plane of the first two canonical covariates, see Figure \ref{fig:rnaFecalColonInterp}. We then color-coded the samples according to the \texttt{treatment} vector. We observed that our estimated canonical covariates clearly separate our samples according to the treatment effect.

\begin{figure}[ht]
\centering
\includegraphics[width=1.2\textwidth, center]{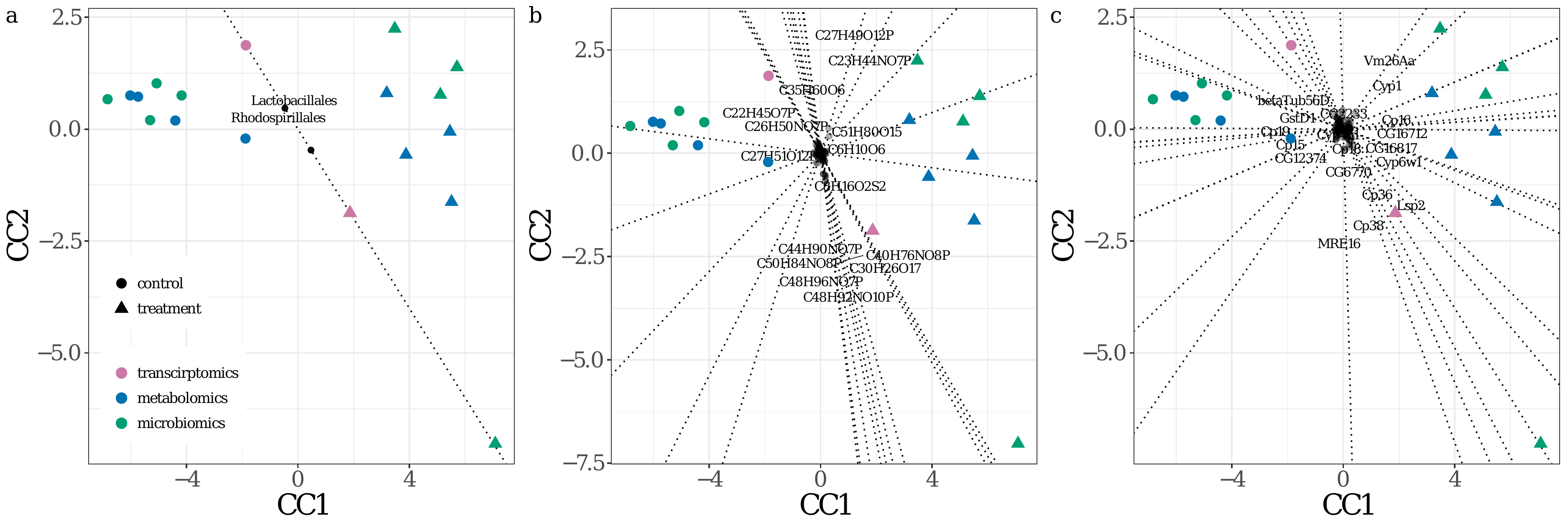}
\caption{Interpolative plots of microbiomics(a), metabolomics(b), and transcriptomics(c) views. Any given sample is interpolated by either the complete parallelogram or the vector sum method explained in Appendix \ref{subsec:app:interplot}}
\label{fig:rnaFecalColonInterp}
\end{figure}

Overall, we see many of the genes and metabolites involved in response to Atrazine identified in the support of the first and second canonical covariates. The fact that many members of individual pathways were returned together is comforting – genes and metabolites in the same or related pathways should co-vary, and they appear to through the lens of our analysis. The novelty and discovery of the sCCA method lies in identifying potential interactions between these pathways – and the current analysis has yielded a number of hypotheses for follow-up studies, including the coupling of germ cell proliferation repression to Linoleic acid metabolism. The identification of genes of unknown function is also interesting – we posit that MRE16 along the second principle axis of variation encodes at least one small functional peptide (e.g. a peptidase or an immunopeptide), and this too will be the subject of future study.

\newpage
\section{Conclusion}
\label{sec:conclusion}

A two-stage approach to sparse CCA problem was introduced, where in the first stage we computed the sparsity patterns of the canonical directions via a fast, convergent concave minimization program. Then we used these sparsity patterns to shrink our problem to a CCA problem of two drastically smaller matrices, where regular CCA methods may be used. We then extended our methods to multi-view settings, i.e. \textit{Multi-View Sparse CCA}, where we have more than two views and also to scenarios where our objective is to generate targeted hypotheses about associations corresponding to a specific experimental design, i.e. \textit{Directed Sparse CCA}. We benchmarked our algorithm and also compared it to several other popular algorithms. Our simulations clearly demonstrated superior solution stability and convergence properties, as well as higher accuracy both in terms of the correlation of the estimated canonical covariates and also in terms of its ability to recover the underlying sparsity patterns of the canonical directions. We also introduced \texttt{MuLe} which is the package implementing our algorithms. We then applied our method to a multi-omic study aiming to understand mechanisms of adaptations of \textit{Drosophila Melanoger (Fruitfly)} to environmental pesticides, here \textit{Atrazine}. Our analysis clearly indicated that the estimated canonical directions, while sparse and interpretable, captures co-variations due to the treatment effect, and also the selected sets of covariates are known, according to the peer-reviewed literature, to be associated with adaptation mechanisms of fruitfly to environmental pesticides and stressors.


\clearpage

\appendix

\section{Proofs}
\label{app:proof}

\subsection{Proof of Proposition \ref{proposition:complexity}}
\label{app:proposition:complexity}
Let's assume without loss of generality that $p_1 \leq p_2$. In this case we can start \texttt{MuLe} to find the sparsity pattern of $\bm{z}_2 \in \mathbb{R}^{p_2}$ first, shrink $\bm{X}_2$ to $\bm{X}_{2red} \in \mathbb{R}^{n \times n_2^{'}}$ where $n_2^{'} \sim n$, then repeat the same for $\bm{z}_1 \in \mathbb{R}^{p_1}$, shrink $\bm{X}_1$ to $\bm{X}_{1red} \in \mathbb{R}^{n \times n_1^{'}}$ where $n_1^{'} \sim n$, and finally compute the first canonical covariates using the shrunken $\bm{X}_{1red}^T\bm{X}_{2red} \in \mathbb{R}^{n_1^{'}\times n_2^{'}}$.

According to the setup of algorithm \ref{alg:2nd}, each iteration to find $\bm{\tau}_2$ is $O(2p_1p_2 + 4p_1 + p_1)$, using $\bm{\tau}_2$ and shrinking $\bm{X}_2$, each iteration for finding $\bm{\tau}_1$ is $O(2p_1n_2^{'} + 4p_1 + n_2^{'})$ which makes the time complexity of both $O(2p_1p_2 + 4p_1 + p_1 + 2p_1n_2^{'} + 4p_1 + n_2^{'})$. With \texttt{pSVDht}, the time complexity of both passes together is $O(4p_1p_2 + 2(p_1 + p_2))$. Assuming $p_2/p_1 = k = o(1)$ and $n \sim n_2^{'}$, if 

$$n < \frac{2kp_1^2 - 2(k+1)p_1 -p_1}{2p_1 +1}$$

The time complexity of \texttt{MuLe} is less than \texttt{pSVDht}. If $p_1 >> 1$, 

$$\frac{2kp_1^2 - 2(k+1)p_1 -p_1}{2p_1 +1} \approx \frac{2kp_1^2 - 2(k+1)p_1 -p_1}{2p_1} = kp_1 - (k+0.5) > p_1$$

So as long as $n < min\{p_1, p_2\}$, our claim stands.

\subsection{Proof of Proposition \ref{proposition:pmd}}
\label{app:proposition:pmd}

Here, just to provide more clarity, Algorithm 3 of \cite{witten:tibshirani:2009} is provided as a representative for the bigger family of \texttt{sSVD} algorithms.

\begin{algorithm}[H] 
 \KwData{Sample Covariance Matrices $\Sigma_{12} = X_1^TX_2$\\  \quad \qquad    $l_1$-penalty parameters $c_1,c_2$}
 \KwResult{ $z_1 \in \mathbb{R}^{p_1}$, $z_2 \in \mathbb{R}^{p_2}$, and $d = z_1^T\Sigma_{12}z_2$}
  Initialize $z_2$ to have $l2-norm$ 1\;
 \While{ convergence criterion is not met }{
	$z_1 \leftarrow \frac{S(\Sigma_{12}z_2, \Delta_1)}{\|S(\Sigma_{12}z_2, \Delta_1)\|_2}$ where $\Delta_1 = 0$ if this results in $\|z_1\|_1 \leq c_1$; otherwise, $\Delta_1$ is chosen to be a positive constant such that $\| z_1\|_1 = c$ \\
	$z_2 \leftarrow \frac{S(\Sigma_{12}^Tz_1, \Delta_2)}{\|S(\Sigma_{12}^Tz_1, \Delta_2)\|_2}$ where $\Delta_2 = 0$ if this results in $\|z_2\|_1 \leq c_2$; otherwise, $\Delta_2$ is chosen to be a positive constant such that $\| z_2\|_1 = c$ \\
    $d \leftarrow z_1^T\Sigma_{12}z_2$
 }
 \caption{ $PMD(L_1, L_1)$ as proposed in \cite{witten:tibshirani:2009}}
 \label{alg:pmd}
\end{algorithm}

 There is no need for a detailed time complexity analysis, as it is evident that although \texttt{MuLe} has order two polynomial time complexity, refer to Appendix \ref{app:proposition:complexity}, the optimization problems in stages 3 and 4 of \textit{PMD}, i.e. finding $\Delta_1$ and $\Delta_2$ that results in $\|z_1\|_1 = c_1$ and $\|z_2\|_1 = c_2$, are of exponential time complexity $O(2^p_1)$ and $O(2^p_2)$. They propose a binary search algorithm for this problem which has less time complexity but doesn't have guaranteed convergence, neither heuristically nor theoretically. In the implementation of the algorithm in the \texttt{PMA} package, the maximum number of iterations is set to a very small number, replacing which with a convergence criteria did not prove to be successful.
 
\section{Complementary Methods and Algorithms}

\subsection{Multi-Factor MuLe}
\label{app:alg:multifactor}


\begin{algorithm}[H] 
 \KwData{Sample Covariance Matrix $\bm{C}_{12}$\\  \quad \qquad    Regularization parameter vectors $\bm{\gamma}_i \in \mathbb{R}^{m}, i \in \{1,2\}$\\ \quad \qquad Initial value vectors $\bm{z}_i \in \mathcal{S}^{p_i}, i \in \{1,2\}$}
 \KwResult{ $\bm{Z}_i \in \mathbb{R}^{p_i \times m}, i \in \{ 1,2 \}$}
Let $\bm{C}_{12}^{(0)} \leftarrow \bm{C}_{12}$\\
\For{$i = 1, \ldots, m$}{

 $(\bm{z}_1^{*(i)}, \bm{Z}_2^{*(i)} \leftarrow sCCA_{MuLe}(\bm{C}_{12}^{(i-1)}, \gamma_{1i}, \gamma_{2i})$\\

$ \bm{C}_{12}^{(i)} = \bm{C}_{12} - \sum_{k = 1}^{i} (\bm{z}_1^{(k)*\top}\bm{C}_{12}^{(k-1)}\bm{z}_2^{(k)*}) \bm{z}_1^{(k)*}\bm{z}_2^{(k)*\top}$

$(\bm{Z}_1[,i], \bm{Z}_2[,i]) \leftarrow (\bm{z}_1^{*(i)}, \bm{z}_2^{*(i)})$
}
\caption{Multi-Factor MuLe}
\label{alg:deflation}
\end{algorithm}

\subsection{Multi-View CCA as Generalized Eigenvalue Problem}
\label{app:mcca}

Here, we frame the CCA problem applied to multiple datasets, $\bm{X}_i$, $i = 1, \ldots, m$, analyzed in \cite{kettenring1971canonical} as the following \textit{Generalized Eigenvalue Problem},

\begin{equation}
\label{eq:gep}
    \begin{bmatrix} \bm{0} & \bm{C}_{12}' & \dots & \bm{C}_{1m}'& \\ \bm{C}_{21}' & \bm{0} &  & \vdots \\
    \vdots & & \ddots & \bm{C}_{(m-1)m}' \\
    \bm{C}_{m1}' & & \bm{C}_{m(m-1)}' & \bm{0}\end{bmatrix}
    \begin{bmatrix}
    \bm{z}_1' \\ \bm{z}_2' \\ \vdots \\ \bm{z}_m'
    \end{bmatrix}
 =
 \lambda
  \begin{bmatrix} \bm{C}_{11}' & \bm{0} & \dots & \bm{0} \\ \bm{0} & \bm{C}_{22}' &  & \vdots \\
    \vdots & & \ddots & \bm{0} \\
    \bm{0} & \dots & \bm{0} & \bm{C}_{mm}'\end{bmatrix}
    \begin{bmatrix}
    \bm{z}_1' \\ \bm{z}_2' \\ \vdots \\ \bm{z}_m'
    \end{bmatrix}
\end{equation}

where $\bm{C}_{ij}'$ is the shrunken $\bm{C}_{ij}$, or the sample covariance matrix of the active entries of $\bm{z}_i$ and $\bm{z}_j$, denoted here as $\bm{z}_i'$ and $\bm{z}_j'$. Equation \ref{eq:gep} can be solved using a wide variety of solvers. We used the \texttt{geigen}\footnote{\url{https://CRAN.R-project.org/package=geigen}} function which is implemented in an r-package of the same name, which uses the routines implemented in \texttt{LAPACK}\footnote{\url{http://github.com/Reference-LAPACK}}. Given that $m$ is usually less than 10, and $\bm{z}_i' = O(n)$, where $n$ is not very large given we're assuming high-dimensional settings, problem \ref{eq:gep} does not involve very large matrices.

\subsection{Multi-View SVD via Power Iteration}
\label{app:msvd}

We proposed a Multi-View CCA in Appendix \ref{app:mcca} which served as the second stage of out two-stage sCCA approach which was to estimate active elements of the canonical directions. Although \ref{eq:gep} is of reasonable size, it still requires inversions which might be deemed as a disadvantage. Although it's very trivial to use ridge regularization to alleviate this issue, here we propose an algorithm which uses power iterations to perform multi-View SVD.

\vspace{\baselineskip}
\begin{algorithm}[H] 
 \KwData{Shrunk Sample Covariance Matrices $\bm{C}_{rs}', \quad 1 \leq r < s \leq m$\\  \quad \qquad Initial values $\bm{z}_r' \in \mathcal{S}^{|\bm{\tau}_r|}, \quad 1 \leq r \leq m$}
 \KwResult{ $\bm{z}_r'$, $r = 1, \ldots , m$, estimated active elements of $\bm{z}_r$}
 initialization\;
 \For{$r = m, \ldots, 1$}{
 \While{ convergence criterion is not met }{
    
    $\bm{z}_r' \leftarrow \sum_{s = 1}^r \bm{C}_{sr}(\bm{C}_{sr}^{\top}\bm{z}_r') + \sum_{s = r + 1}^m \bm{C}_{rs}'\bm{z}_s' $ \\
    $\bm{z}_r \leftarrow \frac{\bm{z}_r }{\| \bm{z}_r \|_2}$
    }
 }
 \caption{\texttt{MuLe} algorithm for optimizing Program \ref{eq:multiznoptsimple}}
 \label{alg:app:3rd}
\end{algorithm}
\vspace{\baselineskip}

\subsection{Two-Stage Directed CCA}
\label{app:semimule}
 
Though simple and obvious, we include this approach in this appendix for the sake of clarity and completeness. Here are the steps for this algorithm.

\begin{enumerate}
    \item Perform variable selection via univariate regression or classification of $\bm{y}$ on each $\bm{X}_i$ resulting in a set of variables, $Q_i$, which are highly associated with the accessory variable.
    \item Subset every datasets such that only the columns selected in the previous steps are kept, resulting in $\bm{X}_i' \in \mathbb{R}^{n \times | Q_i |}$.
    \item Perform sCCA between the datasets using any of the algorithms implemented in \texttt{MuLe}.
\end{enumerate}
 
\section{Further Experimmentations}
 
\subsection{Rank-One Sparse Multi-View CCA Model}
\label{app:subsec:multiview}

To assess the validity of the formulation presented in Program \ref{eq:mCCA} and accuracy of our solution and algorithm presented in Section \ref{subsec:multimodalalg}, for the cases involving more than two, the rank-one model introduced in Section \ref{subsec:rankOne} is extended to three datasets by generating $\mathbf{X}_3$ as follows,

\begin{equation}
\begin{split}
\mathbf{X}_3  = (\mathbf{z}_3 + \epsilon_3)u^T, &\quad \mathbf{z}_3 \in \mathcal{R}^{600}, \quad \epsilon_3 \sim \mathcal{N}(0, 0.1^2), \forall i = 1, \ldots ,600,\\
\mathbf{z}_{1} &= \bigg[\underbrace{1, \ldots ,1}_{25}\quad \underbrace{0, \ldots, 0}_{550}\quad \underbrace{-1, \ldots ,-1}_{25}\bigg]
\end{split}
\end{equation}

where $ \mathbf{u}_i \sim \mathcal{N}(0,1), \forall i = 1, \ldots, 50$.

The coefficient estimates are presented in Figure \ref{fig:rankOneM}. Here, we also included the \texttt{RGCCA} package. Although their conventional sCCA algorithm results were identical to \texttt{PMA}, their generalization to more than two datasets resulted in different and better results. Hence, its inclusion in this simulation. We used each package's own built-in hyper-parameter tuning procedure to find the best parameters. As evident from the results, \texttt{MuLe} identifies the underlying model quite accurately, but \texttt{RGCCA} although does a good job on parameter estimation, it does a very poor job on recovering the sparsity patterns of the canonical directions. \texttt{PMA} misses both critera quite significantly.

In the next section we utilize \texttt{MuLe} to discover correlation structures in a genomic setting.

\begin{figure}
\includegraphics[width=1\textwidth]{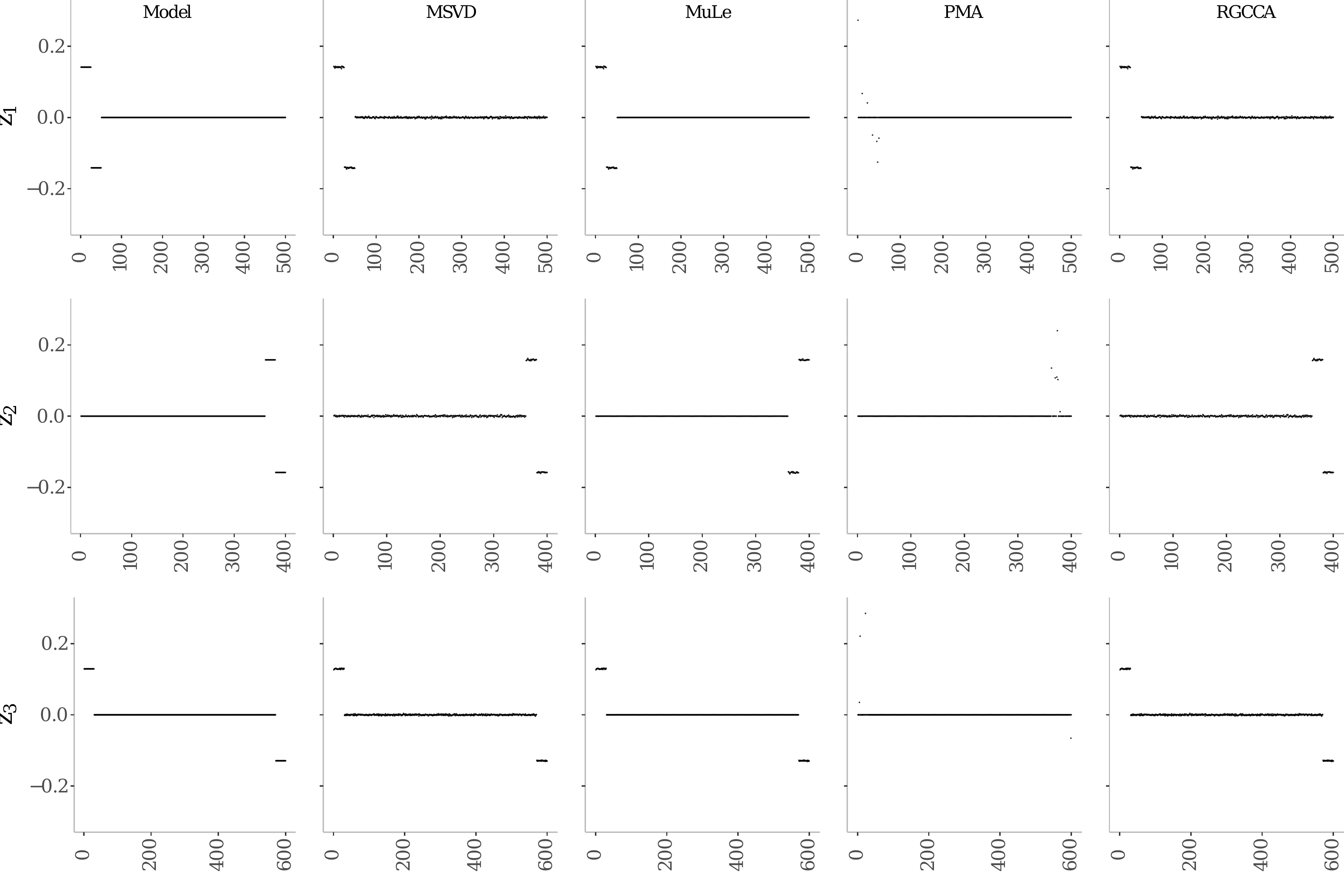}
\caption{Comparing performance of some of the most common multi-view sCCA approaches to that of \textit{MuLe} in recovering the sparsity pattern and estimating active elements of the canonical directions. The \textit{Model} or ``true" canonical directions are plotted in the leftmost plot.}
\label{fig:rankOneM}
\end{figure}

\section{Visualization Methods}
\label{sec:app:vis}

In a general subspace learning problem involving datasets, we're seeking to replace each dataset with three low-dimensional pieces of information, a rule for projecting the original covariates to the learned subspace for the respective subspace, a low-dimensional projection of samples from the original sample-space to the learned sub-space, and a measure of similarity or alignment between the learned subspace. In our linear sCCA context, we replace the dataset $\bm{X}_i$ with $\bm{Z}_i$ whose rows contain the correlation of the covariate $\bm{x}_i$ with the canonical covariates, $\bm{CC}_i$ the projection of samples onto the canonical directions and the canonical correlations $\bm{\rho}_i^{(j)} \in \mathbb{R}^m$ containing the correlation between the $j$-th canonical covariate of the $i$-th dataset and the $j$-th canonical covariates obtained from other datasets. Now we explain the procedures used to create the figures in Section \ref{sec:reda} which facilitate the interpretation of sCCA results. Inspired by the methods proposed in \cite{alves2003interpolative}, we adapt their CCA biplot and interpolative plot to our sCCA settings. In the following brief tutorial, we focus on the first two canonical covariates, thereby keeping only the first two columns of $\bm{Z}_i$ and $\bm{CC}_i$, denoted by $\bm{Z}_i^{(2)}$ and $\bm{CC}_i^{(2)}$, and only $\bm{\rho}_i^{(j)}$ for $j \in \{1,2\}$ and $i = \{1, \ldots, m\}$.

\subsection{CCA Biplot}
\label{subsec:app:biplot}
In order to create the CCA biplot, e.g. Figure \ref{fig:rnaFecalColonBiplot}, we simply plot the first two columns of $\bm{Z}_i^{(2)}$ in the same plot. A key complementary piece of information facilitating interpretation are the first two canonical correlations. Utilizing at this plot, we can form hypotheses about how and to what extend groups of variables from different datasets are associated with each other. The length of the vectors indicate the variable's share in each canonical direction, while the angle between them indicate their degree of association.

\subsection{CCA Interpolative Plots}
\label{subsec:app:interplot}

Another informative visualization we exploit to interpret sCCA results are \textit{Interpolative CCA Plots}, e.g. Figure \ref{fig:rnaFecalColonInterp}. In order to create such figure for each dataset, we first plot $\bm{CC}_i$ from all datasets in the same plot, which by itself provides enlightening insights into how strongly the samples from different datasets align with each other. Next we need to add lines corresponding to the variables from the respective dataset. In order to make interpolation easier and the plots more clear, we first choose a set of marker points $\bm{\mu}_{ij}$ corresponding to the $j$-th variable from the $i$-th dataset, consisting of values within the range of observed values of the variable $\bm{x}_{ij}$, i.e. $\mu_{ijk} \in [ min(\bm{x}_{ij}), max(\bm{x}_{ij}) ]$. We project these points using the following projection $\bm{\mu}_{ij}\bm{e}_{ij}\bm{V}_i^{(2)}$, where $\bm{e}_{ij}$ is a vector whose elements except the $j$-th is zeroed out. Finally, we pass a line through the projected points. Marking the values of each variable corresponding to a sample as a vector along each variable we can find the interpolated position of the said sample. This is a powerful tool as we can find how accurately we can interpolate a samples position using the values of a different dataset. This is specially important in cases where sample matching from different datasets are not exact and samples are matched based on some other metadata, e.g. gender, age etc.

\clearpage

\supplement

\begin{center}
\textbf{\large Supplemental Materials: \mule}
\end{center}

\section{\textsc{MuLe} Package}

An R-implementation of our package \texttt{MuLe}, named \texttt{MuLe-R}, along with the scripts used to perform the simulations and create the visualizations, and the data used in Section \ref{sec:reda} is available online at \url{https://github.com/osolari/MuleR}.

\clearpage
\vskip 0.2in
\bibliography{jmlr}

\begin{thebibliography}{70}
\providecommand{\natexlab}[1]{#1}
\providecommand{\url}[1]{\texttt{#1}}
\expandafter\ifx\csname urlstyle\endcsname\relax
  \providecommand{\doi}[1]{doi: #1}\else
  \providecommand{\doi}{doi: \begingroup \urlstyle{rm}\Url}\fi

\bibitem[Akaho(2001)]{akaho01}
S.~Akaho.
\newblock A kernel method for canonical correlation analysis.
\newblock \emph{In Proceedings of the International Meeting of the Psychometric
  Society}, 2001.

\bibitem[Alam et~al.(2008)Alam, Nasser, and Fukumizu]{alam08}
Md.~A. Alam, M.~Nasser, and K.~Fukumizu.
\newblock Sensitivity analysis in robust and kernel canonical correlation
  analysis.
\newblock \emph{11th International Conference on Computer and Information
  Technology}, 0:\penalty0 399--404, 2008.

\bibitem[Alves and Oliveira(2003)]{alves2003interpolative}
M~Rui Alves and M~Beatriz Oliveira.
\newblock Interpolative biplots applied to principal component analysis and
  canonical correlation analysis.
\newblock \emph{Journal of Chemometrics: A Journal of the Chemometrics
  Society}, 17\penalty0 (11):\penalty0 594--602, 2003.

\bibitem[Andrew et~al.(2013)Andrew, Arora, Bilmes, and Livescu]{andrew13}
G.~Andrew, R.~Arora, J.~Bilmes, and K.~Livescu.
\newblock Deep canonical correlation analysis.
\newblock \emph{International Conference on Machine Learning}, pages
  1247--1255, 2013.

\bibitem[Arab et~al.(2006)Arab, Rossary, Soulere, and
  Steghens]{arab2006conjugated}
Khelifa Arab, Adrien Rossary, Laurent Soulere, and Jean-Paul Steghens.
\newblock Conjugated linoleic acid, unlike other unsaturated fatty acids,
  strongly induces glutathione synthesis without any lipoperoxidation.
\newblock \emph{British Journal of Nutrition}, 96\penalty0 (5):\penalty0
  811--819, 2006.

\bibitem[Bach and Jordan(2002)]{bach02}
F.R. Bach and M.I. Jordan.
\newblock Kernel independent component analysis.
\newblock \emph{Journal of machine learning research}, pages 1--48, 2002.

\bibitem[Bach and Jordan(2005)]{bach05}
F.R. Bach and M.I. Jordan.
\newblock A probabilistic interpretation of canonical correlation analysis.
\newblock \emph{Technical Report}, 2005.

\bibitem[Baur and Bozdag(2015)]{baur15}
B.~Baur and S.~Bozdag.
\newblock A canonical correlation analysis-based dynamic bayesian network prior
  to infer gene regulatory networks from multiple types of biological data.
\newblock \emph{Journal of Computational Biology}, 22\penalty0 (4):\penalty0
  289--299, 2015.

\bibitem[Benson(1995)]{benson1995concave}
Harold~P Benson.
\newblock Concave minimization: theory, applications and algorithms.
\newblock In \emph{Handbook of global optimization}, pages 43--148. Springer,
  1995.

\bibitem[Blaschko et~al.(2008)Blaschko, Lampert, and Gretton]{blaschko08}
M.B. Blaschko, C.H. Lampert, and A.~Gretton.
\newblock Semi-supervised laplacian regularization of kernel canonical
  correlation analysis.
\newblock \emph{Joint European Conference on Machine Learning and Knowledge
  Discovery in Databases}, 0:\penalty0 133--145, 2008.

\bibitem[Brown et~al.(2014)Brown, Boley, Eisman, May, Stoiber, Duff, Booth,
  Wen, Park, Suzuki, et~al.]{brown2014diversity}
James~B Brown, Nathan Boley, Robert Eisman, Gemma~E May, Marcus~H Stoiber,
  Michael~O Duff, Ben~W Booth, Jiayu Wen, Soo Park, Ana~Maria Suzuki, et~al.
\newblock Diversity and dynamics of the drosophila transcriptome.
\newblock \emph{Nature}, 512\penalty0 (7515):\penalty0 393, 2014.

\bibitem[Cai(2013)]{cai13}
J.~Cai.
\newblock The distance between feature subspaces of kernel canonical
  correlation analysis.
\newblock \emph{Mathematical and Computer Modelling}, 3:\penalty0 970--975,
  2013.

\bibitem[Campos and Colbourne(2018)]{campos2018omics}
Bruno Campos and John~K Colbourne.
\newblock How omics technologies can enhance chemical safety regulation:
  perspectives from academia, government, and industry: The perspectives column
  is a regular series designed to discuss and evaluate potentially competing
  viewpoints and research findings on current environmental issues.
\newblock \emph{Environmental toxicology and chemistry}, 37\penalty0
  (5):\penalty0 1252, 2018.

\bibitem[Cao et~al.(2015)Cao, Ju, Li, Jian, and Jiang]{cao15}
L.~Cao, Z.~Ju, J.~Li, R.~Jian, and C.~Jiang.
\newblock Sequence detection analysis based on canonical correlation for
  steady-state visual evoked potential brain computer interfaces.
\newblock \emph{Journal of neuroscience methods}, 0\penalty0 (253):\penalty0
  10–17, 2015.

\bibitem[Chandler et~al.(2011)Chandler, Lang, Bhatnagar, Eisen, and
  Kopp]{chandler2011bacterial}
James~Angus Chandler, Jenna~Morgan Lang, Srijak Bhatnagar, Jonathan~A Eisen,
  and Artyom Kopp.
\newblock Bacterial communities of diverse drosophila species: ecological
  context of a host--microbe model system.
\newblock \emph{PLoS genetics}, 7\penalty0 (9):\penalty0 e1002272, 2011.

\bibitem[Cichonska et~al.(2016)Cichonska, Rousu, Marttinen, Kangas, Soininen,
  Lehtim\"aki, Raitakari, J\"arvelin, Salomaa, Ala-Korpela, and
  others.]{cichonska16}
A.~Cichonska, J.~Rousu, P.~Marttinen, A.J. Kangas, P.~Soininen, T.~Lehtim\"aki,
  O.T. Raitakari, M.R. J\"arvelin, V.~Salomaa, M~Ala-Korpela, and others.
\newblock metacca: Summary statistics-based multivariate meta-analysis of
  genome-wide association studies using canonical correlation analysis.
\newblock \emph{Bioinformatics}, 32:\penalty0 1981--9, 2016.

\bibitem[Dunham and Kravetz(1975)]{dunham75}
R.B. Dunham and D.J. Kravetz.
\newblock Canonical correlation analysis in a predictive system.
\newblock \emph{The Journal of Experimental Education}, 43\penalty0
  (4):\penalty0 35--42, 1975.

\bibitem[d’Aspremont et~al.(2008)d’Aspremont, Bach, and
  Ghaoui]{d2008optimal}
Alexandre d’Aspremont, Francis Bach, and Laurent~El Ghaoui.
\newblock Optimal solutions for sparse principal component analysis.
\newblock \emph{Journal of Machine Learning Research}, 9\penalty0
  (Jul):\penalty0 1269--1294, 2008.

\bibitem[Ewerbring and Luk(1989)]{ewerbring89}
L.M. Ewerbring and F.T. Luk.
\newblock Canonical correlations and generalized svd: applications and new
  algorithms.
\newblock \emph{In 32nd Annual Technical Symposium, International Society for
  Optics and Photonics}, page 206–222, 1989.

\bibitem[Fang et~al.(2016)Fang, Lin, S.C.~Schulz, Calhoun, , and Wang]{fang16}
J.~Fang, D.~Lin, Z.~Xu S.C.~Schulz, V.D. Calhoun, , and Y.P. Wang.
\newblock Joint sparse canonical correlation analysis for detecting
  differential imaging genetics modules.
\newblock \emph{Bioinformatics}, 32\penalty0 (22):\penalty0 3480--3488, 2016.

\bibitem[Friman et~al.(2001)Friman, Cedefamn, Lundberg, Borga, and
  Knutsson]{fMRI}
O.~Friman, J.~Cedefamn, P.~Lundberg, M.~Borga, and H.~Knutsson.
\newblock Detection of neural activity in functional mri using canonical
  correlation analysis.
\newblock \emph{Magnetic Resonance in Medicine}, 45:\penalty0 323--330, 2001.

\bibitem[Gestel et~al.(2001)Gestel, Suykens, Brabanter, Moor, and
  Vandewalle]{gestel01}
T.~Van Gestel, J.A.K. Suykens, J.~De Brabanter, B.~De Moor, and J.~Vandewalle.
\newblock Kernel canonical correlation analysis and least squares support
  vector machines.
\newblock \emph{International Conference on Artificial Neural Networks.}, pages
  384--389, 2001.

\bibitem[Hardoon et~al.(2004)Hardoon, Szedmak, and
  Shawe-Taylor]{hardoon2004canonical}
David~R Hardoon, Sandor Szedmak, and John Shawe-Taylor.
\newblock Canonical correlation analysis: An overview with application to
  learning methods.
\newblock \emph{Neural computation}, 16\penalty0 (12):\penalty0 2639--2664,
  2004.

\bibitem[Hardoon and Shawe-Taylor(2009)]{hardoon09}
D.R. Hardoon and J.~Shawe-Taylor.
\newblock Convergence analysis of kernel canonical correlation analysis: theory
  and practice.
\newblock \emph{Machine learning.}, 1:\penalty0 23--38, 2009.

\bibitem[Hardoon and Shawe-Taylor(2011)]{hardoon11}
D.R. Hardoon and J.~Shawe-Taylor.
\newblock Sparse canonical correlation analysis.
\newblock \emph{Machine Learning}, 3:\penalty0 331--353, 2011.

\bibitem[Healy(1957)]{healy57}
M.J.R. Healy.
\newblock A rotation method for computing canonical correlations.
\newblock \emph{Math. Comp.}, 58:\penalty0 83--86, 1957.

\bibitem[Heij and Roorda(1991)]{heij91}
C.~Heij and B.~Roorda.
\newblock A modified canonical correlation approach to approximate state space
  modeling.
\newblock \emph{Proceedings of the 30th IEEE Conference on Decision and
  Control}, pages 1343--1348, 1991.

\bibitem[Hol{\'a}skov{\'a} et~al.(2019)Hol{\'a}skov{\'a}, Elliott, Brundage,
  Lukomska, Schafer, and Barnett]{holaskova2019long}
Ida Hol{\'a}skov{\'a}, Meenal Elliott, Kathleen Brundage, Ewa Lukomska, Rosana
  Schafer, and John~B Barnett.
\newblock Long-term immunotoxic effects of oral prenatal and neonatal atrazine
  exposure.
\newblock \emph{Toxicological Sciences}, 168\penalty0 (2):\penalty0 497--507,
  2019.

\bibitem[Hopkins(1969)]{hopkins69}
C.E. Hopkins.
\newblock Statistical analysis by canonical correlation: a computer
  application.
\newblock \emph{Health services research}, 4\penalty0 (4):\penalty0 304, 1969.

\bibitem[Hotelling(1935)]{hotelling}
H.~Hotelling.
\newblock The most predictable criterion.
\newblock \emph{Journal of Educational Psychology}, 26:\penalty0 139--142,
  1935.

\bibitem[Hyman et~al.(2002)Hyman, Kauraniemi, Hautaniemi, Wolf, Mousses,
  Rozen-blum, Ringner, Sauter, Monni, Elkahloun, Kallioniemi, and
  Kallioniemi]{hyman2002}
E.~Hyman, P.~Kauraniemi, S.~Hautaniemi, M.~Wolf, S.~Mousses, E.~Rozen-blum,
  M.~Ringner, G.~Sauter, O.~Monni, A.~Elkahloun, O.-P. Kallioniemi, and
  A.~Kallioniemi.
\newblock Impact of dna amplication on gene expression patterns in breast
  cancer.
\newblock \emph{Cancer Research}, 0\penalty0 (62):\penalty0 6240--6245, 2002.

\bibitem[Journ\'ee et~al.(2010)Journ\'ee, Nesterov, Richtr\'arik, and
  Sepulchre]{journe:nesterov}
M.~Journ\'ee, Y.~Nesterov, P.~Richtr\'arik, and R.~Sepulchre.
\newblock Generalized power method for sparse principal component analysis.
\newblock \emph{Journal of Machine Learning Research}, 11:\penalty0 517--553,
  2010.

\bibitem[Kettenring(1971)]{kettenring1971canonical}
Jon~R Kettenring.
\newblock Canonical analysis of several sets of variables.
\newblock \emph{Biometrika}, 58\penalty0 (3):\penalty0 433--451, 1971.

\bibitem[Klami et~al.(2012)Klami, Virtanen, and Kaski]{klami12}
A.~Klami, S.~Virtanen, and S.~Kaski.
\newblock Bayesian exponential family projections for coupled data sources.
\newblock \emph{arXiv:1203.3489}, 2012.

\bibitem[Lai and Fyfe(1999)]{lai99}
P.L. Lai and C.~Fyfe.
\newblock A neural implementation of canonical correlation analysis.
\newblock \emph{Neural Networks}, 10:\penalty0 1391--1397, 1999.

\bibitem[Lai and Fyfe(2000)]{lai00}
P.L. Lai and C.~Fyfe.
\newblock Kernel and nonlinear canonical correlation analysis.
\newblock \emph{International Journal of Neural Systems}, 10:\penalty0
  365--377, 2000.

\bibitem[Larson et~al.(2014)Larson, Jenkins, Larson, Vierkant, Sellers, Phelan,
  Schildkraut, Sutphen, Pharoah, Gayther, et~al.]{larson14}
N.B. Larson, G.D. Jenkins, M.C. Larson, R.A. Vierkant, T.A. Sellers, C.M.
  Phelan, J.M. Schildkraut, R.~Sutphen, P.P.D. Pharoah, S.~A. Gayther, et~al.
\newblock Kernel canonical correlation analysis for assessing gene–gene
  interactions and application to ovarian cancer.
\newblock \emph{European Journal of Human Genetics}, 1:\penalty0 126--131,
  2014.

\bibitem[Lindsey et~al.(1985)Lindsey, Webster, , and Halper]{lindsey85}
H.~Lindsey, J.T. Webster, , and S.~Halper.
\newblock Canonical correlation as a discriminant tool in a periodontal
  problem.
\newblock \emph{Biometrical journal}, 3\penalty0 (27):\penalty0 257--264, 1985.

\bibitem[Mangasarian(1996)]{mangasarian1996machine}
OL~Mangasarian.
\newblock Machine learning via polyhedral concave minimization.
\newblock In \emph{Applied Mathematics and Parallel Computing}, pages 175--188.
  Springer, 1996.

\bibitem[Melzer et~al.(2001)Melzer, Reiter, and Bischof]{melzer01}
T.~Melzer, M.~Reiter, and H.~Bischof.
\newblock Nonlinear feature extraction using generalized canonical correlation
  analysis.
\newblock \emph{International Conference on Artificial Neural Networks.},
  0:\penalty0 353--360, 2001.

\bibitem[Mercer(1909)]{mercer1909xvi}
James Mercer.
\newblock Xvi. functions of positive and negative type, and their connection
  the theory of integral equations.
\newblock \emph{Philosophical transactions of the royal society of London.
  Series A, containing papers of a mathematical or physical character},
  209\penalty0 (441-458):\penalty0 415--446, 1909.

\bibitem[Monmonier and Finn(1973)]{monmonier73}
M.S. Monmonier and F.E. Finn.
\newblock Improving the interpretation of geographical canonical correlation
  models.
\newblock \emph{The Professional Geographer}, 25:\penalty0 140--142, 1973.

\bibitem[Morley et~al.(2004)Morley, Molony, Weber, Devlin, Ewens, Spielman, and
  Cheung]{morley2004}
M.~Morley, C.~Molony, T.~Weber, J.~Devlin, K.~Ewens, R.~Spielman, and
  V.~Cheung.
\newblock Genetic analysis of genome-wide variation in human gene expression.
\newblock \emph{Nature}, 0\penalty0 (430):\penalty0 743--747, 2004.

\bibitem[Nakanishi et~al.(2015)Nakanishi, Wang, Wang, and Jung]{nakanishi15}
M.~Nakanishi, Y.~Wang, Y.T Wang, and T.P. Jung.
\newblock A comparison study of canonical correlation analysis based methods
  for detecting steady-state visual evoked potentials.
\newblock \emph{PloS one}, 10\penalty0 (10):\penalty0 10–17, 2015.

\bibitem[Ogura et~al.(2013)Ogura, Fujikoshi, and Sugiyama]{ogura13}
T.~Ogura, Y.~Fujikoshi, and T.~Sugiyama.
\newblock A variable selection criterion for two sets of principal component
  scores in principal canonical correlation analysis.
\newblock \emph{Communications in Statistics-Theory and Methods}, 42\penalty0
  (12):\penalty0 2118--2135, 2013.

\bibitem[Orsini et~al.(2018)Orsini, Brown, Shams~Solari, Li, He, Podicheti,
  Stoiber, Spanier, Gilbert, Jansen, et~al.]{orsini2018early}
Luisa Orsini, James~B Brown, Omid Shams~Solari, Dong Li, Shan He, Ram
  Podicheti, Marcus~H Stoiber, Katina~I Spanier, Donald Gilbert, Mieke Jansen,
  et~al.
\newblock Early transcriptional response pathways in daphnia magna are
  coordinated in networks of crustacean-specific genes.
\newblock \emph{Molecular ecology}, 27\penalty0 (4):\penalty0 886--897, 2018.

\bibitem[{Ouarda} et~al.(2001){Ouarda}, {Girard}, {Cavadias}, and
  {Bob{\'e}e}]{2001JHyd..254..157O}
T.~B.~M.~J. {Ouarda}, C.~{Girard}, G.~S. {Cavadias}, and B.~{Bob{\'e}e}.
\newblock Regional flood frequency estimation with canonical correlation
  analysis.
\newblock \emph{Journal of Hydrology}, 254:\penalty0 157--173, December 2001.
\newblock \doi{10.1016/S0022-1694(01)00488-7}.

\bibitem[Parkhomenko et~al.(2007)Parkhomenko, Tritchler, and
  Beyene]{parkhomenkoGWSCCA}
E.~Parkhomenko, D.~Tritchler, and J.~Beyene.
\newblock Genome-wide sparse canonical correlation of gene expression with
  genotypes.
\newblock \emph{BMC Proceedings}, 1:\penalty0 s119, 2007.

\bibitem[Parkhomenko et~al.(2009)Parkhomenko, Tritchler, and
  Beyene]{parkhomenkoSCCA}
E.~Parkhomenko, D.~Tritchler, and J.~Beyene.
\newblock Sparse canonical correlation analysis with application to genomic
  data integration.
\newblock \emph{Statistical Applications in Genetics and Molecular Biology},
  8:\penalty0 1--34, 2009.

\bibitem[Pollack et~al.(2002)Pollack, Sorlie, Perou, Rees, Jerey, Lonning,
  Tibshi-rani, Botstein, Borresen-Dale, and Brown]{pollack2002}
J.~Pollack, T.~Sorlie, C.~Perou, C.~Rees, S.~Jerey, P.~Lonning, R.~Tibshi-rani,
  D.~Botstein, A.~Borresen-Dale, and P.~Brown.
\newblock Microarray analysis reveals a major direct role of dna copy number
  alteration in the transcriptional program of human breast tumors.
\newblock \emph{Proceedings of the National Academy of Sciences}, 0\penalty0
  (99):\penalty0 12963--12968, 2002.

\bibitem[Rousu et~al.(2013)Rousu, Agranoff, Sodeinde, Shawe-Taylor, and
  Fernandez-Reyes]{rousu13}
J.~Rousu, D.D. Agranoff, O.~Sodeinde, J.~Shawe-Taylor, and D.~Fernandez-Reyes.
\newblock Biomarker discovery by sparse canonical correlation analysis of
  complex clinical phenotypes of tuberculosis and malaria.
\newblock \emph{PLoS Comput Biol}, 9\penalty0 (4), 2013.

\bibitem[Sarkar and Chakraborty(2015)]{sarkar15}
B.K. Sarkar and C.~Chakraborty.
\newblock Dna pattern recognition using canonical correlation algorithm.
\newblock \emph{Journal of biosciences}, 40\penalty0 (4):\penalty0 709--719,
  2015.

\bibitem[Schell and Gardner(1995)]{schell95}
S.V. Schell and W.A. Gardner.
\newblock Programmable canonical correlation analysis: A flexible framework for
  blind adaptive spatial filtering.
\newblock \emph{IEEE transactions on signal processing}, 43\penalty0
  (12):\penalty0 2898--2908, 1995.

\bibitem[Sengupta et~al.(2015)Sengupta, Litoff, and Baldwin]{sengupta2015hr96}
Namrata Sengupta, Elizabeth~J Litoff, and William~S Baldwin.
\newblock The hr96 activator, atrazine, reduces sensitivity of d. magna to
  triclosan and dha.
\newblock \emph{Chemosphere}, 128:\penalty0 299--306, 2015.

\bibitem[Seoane et~al.(2014)Seoane, Campbell, Day, Casas, and Gaunt]{seoane14}
J.A. Seoane, C.~Campbell, I.N.M. Day, J.P. Casas, and T.R. Gaunt.
\newblock Canonical correlation analysis for genebased pleiotropy discovery.
\newblock \emph{PLoS Comput Biol}, 10\penalty0 (10), 2014.

\bibitem[Shen et~al.(2013)Shen, Sun, and Yuan]{shen13}
X-B Shen, Q-S Sun, and Y-H Yuan.
\newblock Orthogonal canonical correlation analysis and its application in
  feature fusion.
\newblock \emph{16th International Conference on Information Fusion}, pages
  151--157, 2013.

\bibitem[Sieber and Thummel(2009)]{sieber2009dhr96}
Matthew~H Sieber and Carl~S Thummel.
\newblock The dhr96 nuclear receptor controls triacylglycerol homeostasis in
  drosophila.
\newblock \emph{Cell metabolism}, 10\penalty0 (6):\penalty0 481--490, 2009.

\bibitem[Simonson et~al.(1983)Simonson, Stowe, and Watson]{assetLiability}
D.~Simonson, J.~Stowe, and C.~Watson.
\newblock A canonical correlation analysis of commercial bank asset/liability
  structures.
\newblock \emph{Journal of Financial and Quantitative Analysis}, 10:\penalty0
  125--140, 1983.

\bibitem[Snijders et~al.(2017)Snijders, Langley, Kim, Brislawn, Noecker, Zink,
  Fansler, Casey, Miller, Huang, et~al.]{snijders2017influence}
Antoine~M Snijders, Sasha~A Langley, Young-Mo Kim, Colin~J Brislawn, Cecilia
  Noecker, Erika~M Zink, Sarah~J Fansler, Cameron~P Casey, Darla~R Miller,
  Yurong Huang, et~al.
\newblock Influence of early life exposure, host genetics and diet on the mouse
  gut microbiome and metabolome.
\newblock \emph{Nature microbiology}, 2\penalty0 (2):\penalty0 16221, 2017.

\bibitem[Tu et~al.(1989)Tu, Burdick, Millican, and McGown]{tu89}
X.M. Tu, D.S. Burdick, D.W. Millican, and L.B. McGown.
\newblock Canonical correlation technique for rank estimation of
  excitation-emission matrices.
\newblock \emph{Analytical Chemistry}, 19\penalty0 (61):\penalty0 2219--2224,
  1989.

\bibitem[Vinod(1976)]{vinod76}
H.D. Vinod.
\newblock Canonical ridge and econometrics of joint production.
\newblock \emph{Journal of Econometrics}, 4:\penalty0 147–166, 1976.

\bibitem[Waaijenborg et~al.(2008)Waaijenborg, Verselewel~de Witt~Hamer, and
  Zwinderman]{waaijenborg}
S.~Waaijenborg, P.~Verselewel~de Witt~Hamer, and A.~Zwinderman.
\newblock Quantifying the association between gene expressions and dna-markers
  by penalized canonical correlation analysis.
\newblock \emph{Statistical Applications in Genetics and Molecular Biology}, 7,
  2008.

\bibitem[Wang et~al.(2013)Wang, Lin, and Zhang]{wang13}
G.C. Wang, N.~Lin, and B.~Zhang.
\newblock Dimension reduction in functional regression using mixed data
  canonical correlation analysis.
\newblock \emph{Stat Interface}, 6:\penalty0 187--196, 2013.

\bibitem[Wang et~al.(2016)Wang, Yan, Lee, and Livescu]{wang2016deep}
Weiran Wang, Xinchen Yan, Honglak Lee, and Karen Livescu.
\newblock Deep variational canonical correlation analysis.
\newblock \emph{arXiv preprint arXiv:1610.03454}, 2016.

\bibitem[Waugh(1942)]{waugh42}
F.V. Waugh.
\newblock Regressions between sets of variables.
\newblock \emph{Econometrica, Journal of the Econometric Society}, page
  290–310, 1942.

\bibitem[Wiesel et~al.(2008)Wiesel, Kliger, and Hero~III]{wiesel2008greedy}
Ami Wiesel, Mark Kliger, and Alfred~O Hero~III.
\newblock A greedy approach to sparse canonical correlation analysis.
\newblock \emph{arXiv preprint arXiv:0801.2748}, 2008.

\bibitem[Witten and Tibshirani(2009)]{witten:tibshirani:2009}
D.~Witten and R.~Tibshirani.
\newblock Extensions of sparse canonical correlation analysis with applications
  to genomic data.
\newblock \emph{Statistical Applications in Genomics and Molecular Biology}, 8,
  2009.

\bibitem[Witten et~al.(2009)Witten, Tibshirani, and Hastie]{wittentibhastie09}
D.M. Witten, R.~Tibshirani, and T.~Hastie.
\newblock A penalized matrix decomposition, with applications to sparse
  principal components and canonical correlation analysis.
\newblock \emph{Biostatistics}, 3:\penalty0 515--534, 2009.

\bibitem[Wong et~al.(1980)Wong, Fung, and Lau]{wong80}
K.W. Wong, P.C.W. Fung, and C.C. Lau.
\newblock Study of the mathematical approximations made in the basis
  correlation method and those made in the canonical-transformation method for
  an interacting bose gas.
\newblock \emph{Physical Review}, 3\penalty0 (22):\penalty0 1272, 1980.

\bibitem[Yamanishi et~al.(2003)Yamanishi, Vert, Nakaya, and
  Kanehisa]{genomicYamanishi}
Y.~Yamanishi, J.P. Vert, A.~Nakaya, and M.~Kanehisa.
\newblock Extraction of correlated gene clusters from multiple genomic data by
  generalized kernel canonical correlation analysis.
\newblock \emph{Bioinformatics}, 19:\penalty0 i323--i330, 2003.

\end{thebibliography}

\end{document}